%% file: neurips_2026.tex
\documentclass{article}

\usepackage[main, final]{neurips_2026}

\usepackage[utf8]{inputenc}
\usepackage[T1]{fontenc}
\usepackage{hyperref}
\usepackage{url}
\usepackage{booktabs}
\usepackage{amsfonts}
\usepackage{nicefrac}
\usepackage{microtype}

\usepackage{amsmath}
\usepackage{amsthm}
\theoremstyle{plain}
\newtheorem{theorem}{Theorem}
\newtheorem{proposition}{Proposition}

\theoremstyle{definition}

\usepackage{algorithm,algorithmicx,algpseudocode}
\usepackage[table]{xcolor}
\usepackage{comment}
\usepackage[font=small]{caption}
\usepackage{multirow}
\usepackage{tabularx}
\usepackage{threeparttable}
\usepackage{pifont}
\newcommand{\cmark}{\ding{51}}
\newcommand{\xmark}{\ding{53}}
\usepackage{graphicx}
\usepackage{subcaption}
\usepackage{wrapfig}
\usepackage{titletoc}
\usepackage{array}
\usepackage{placeins}

\title{FiTS: Interpretable Spiking Neurons \\ via Frequency Selectivity and Temporal Shaping}

\author{
  Jongmin Choi \quad Joon Son Chung \\
  Korea Advanced Institute of Science and Technology (KAIST)\\
}

\begin{document}

\maketitle

\begin{abstract}
\input{sections/abstract}
\end{abstract}

\input{sections/introduction}
\input{sections/related_work}
\input{sections/methods}
\input{sections/experiments}
\input{sections/discussions}
\input{sections/conclusion}

\begin{ack}
This work was supported by the National Research Foundation of Korea (NRF) grant funded by the Korea government (MSIT) (No. RS-2023-00212845, Multimodal Speech Processing for Human-Computer Interaction).
\end{ack}

\bibliography{shortstrings,bibliography}
\bibliographystyle{plain}

\clearpage
\input{sections/appendix}

\end{document}

%% file: sections/abstract.tex
Spiking Neural Networks (SNNs) are a promising framework for event-driven temporal processing. Prior work has improved temporal modeling through richer neuron dynamics and network-level mechanisms such as recurrence and delays, but it remains unclear how individual spiking neurons should specialize within a network. In this work, we introduce \textbf{FiTS}, a spiking neuron that factorizes temporal computation within each neuron into Frequency Selectivity (FS) and Temporal Shaping (TS). The FS module parameterizes each neuron's target frequency as the maximizer of its subthreshold magnitude response, while the TS module reshapes when frequency components contribute to membrane voltage accumulation through group-delay modulation. On auditory benchmarks where frequency selectivity and timing are central to the input structure, FiTS consistently improves over a plain Leaky Integrate-and-Fire (LIF) baseline in simple feedforward SNNs without recurrence or network-level delays, while remaining competitive with strong temporal SNN baselines. Beyond accuracy, the learned target frequencies and group-delay shifts provide interpretable neuron-level summaries of the frequency and timing organization learned within the network. Our code is available at \url{https://github.com/kaistmm/FiTS}.

%% file: sections/introduction.tex
\section{Introduction}

Spiking Neural Networks (SNNs) are widely regarded as the third generation of neural networks~\cite{MAASS19971659}. They represent information through spike trains over time, and their event-driven sparse computation makes them a promising framework for efficient processing of temporally structured data such as audio and event streams, particularly on neuromorphic hardware~\cite{8259423, gallego2020event, Roy2019TowardsSM}. Although recent work has substantially improved SNN performance on temporal tasks~\cite{baronig2024advancing, chen2024pmsn, fabre2025structured, hammouamri2024learning, huber2024scaling, queant2025delrec}, it remains unclear how to parameterize the temporal role of an individual spiking neuron within a network.

Prior work has improved temporal modeling along two main directions. At the neuron level, adaptive and resonant dynamics enrich temporal responses~\cite{baronig2024advancing, fabre2025structured, huang2024prf, huber2024scaling, IZHIKEVICH2001883, zhang2025lif}. At the network level, recurrence, learnable synaptic delays, and broader architectural inductive biases strengthen temporal representations~\cite{10.3389/fnins.2022.865897, deckers2024co, hammouamri2024learning, queant2025delrec}. These advances are effective, but they often leave each neuron's frequency preference and response timing implicit in learned coefficients or network mechanisms.

In this work, we introduce FiTS, a spiking neuron that makes these roles explicit by factorizing neuron-level temporal computation into Frequency Selectivity (FS) and Temporal Shaping (TS). Motivated by intrinsic neuronal resonance, where neurons respond selectively to inputs at preferred frequencies~\cite{HUTCHEON2000216}, the FS module parameterizes each neuron's target frequency as the maximizer of its subthreshold magnitude response. The TS module uses group-delay modulation within the neuron to reshape when frequency components contribute to pre-spike membrane voltage accumulation. After training, these parameters can be read as parameter-level summaries of the learned frequency and timing organization. Our main contributions are summarized as follows.

\begin{itemize}
\item We introduce the FS module, which maps a response-level target frequency to the adaptation strength through a closed-form inverse, enabling frequency-domain initialization, learning, and post-training interpretation within the same coordinate.
\item We introduce the TS module, which reshapes pre-spike membrane voltage accumulation through group-delay modulation. Unlike network-level delays that change when emitted spikes reach downstream neurons, the TS module controls when frequency-specific input responses contribute to spike generation within a neuron.
\item We show that our FiTS neuron consistently improves simple feedforward SNNs across auditory benchmarks and yields learned target frequencies and group-delay shifts that summarize the network's learned organization of frequency and timing information.
\end{itemize}

%% file: sections/related_work.tex
\section{Related Work}
Early work on training deep SNNs addressed the differentiability problem caused by spike generation and reset using surrogate gradients~\cite{NEURIPS2018_c203d8a1, NEURIPS2018_82f2b308, zenke2018superspike}. Subsequent methods improved temporal learning by reshaping training objectives, simplifying backpropagation through time, or facilitating temporal gradient propagation~\cite{deng2022temporal, guo2024take, meng2023towards, xiao2022online}. Neuron-pathway designs also improve long-range credit assignment or temporal-gradient flow~\cite{fang2023parallel, huang2024clif, zhang2024tc}. These methods improve how SNNs are optimized through time, but they do not directly parameterize a neuron's frequency preference or response timing.

Another line of work improves temporal processing through network-level mechanisms. Recurrent spiking architectures with adaptive neurons extend effective memory across long temporal sequences~\cite{NEURIPS2018_c203d8a1, 10.3389/fnins.2022.865897, yin2021accurate}, while learnable synaptic or axonal delays adjust spike arrival times to improve temporal alignment in feedforward or recurrent architectures~\cite{deckers2024co, hammouamri2024learning, queant2025delrec, sun2022axonal}. More recently, delay variables have also been introduced inside neuron state dynamics to provide each neuron with access to a finite temporal input history~\cite{karilanova2025delays}. Frequency-domain attention modules such as FSTA-SNN~\cite{yu2025fsta} reweight spike feature maps using DCT-based spatial-temporal attention, but operate as feature-level modules rather than intrinsic neuron-level frequency-response parameterizations. Thus, these methods improve temporal routing, memory, or feature reweighting, but they do not directly parameterize the group-delay structure by which frequency-specific subthreshold responses accumulate before spike generation.

A parallel direction enriches single-neuron temporal dynamics. Adaptive LIF-type neurons, learnable time-constant variants, and gated LIF models introduce richer timescales or adaptation dynamics~\cite{baronig2024advancing, fang2021incorporating, yao2022glif, zhang2025lif}. Resonate-and-fire and structured state-space spiking models make oscillatory dynamics or frequency-sensitive regimes more accessible to learning~\cite{fabre2025structured, 10.5555/3692070.3692805, huang2024prf, huber2024scaling, IZHIKEVICH2001883, shen2025spikingssms}. Structural neuron models further expand temporal computation through multiple branches, compartments, or parallel internal pathways~\cite{chen2024pmsn, wang-etal-2025-mmdend, zhang2026dendritic, zheng2024dendritic}. When frequency preference appears in these models, it is typically an implicit consequence of learned dynamical or state-transition parameters, rather than a response-level target frequency directly exposed to optimization. Classical impedance-based analyses characterize subthreshold resonance by the input frequency that maximizes the membrane-voltage response for a given neuron model~\cite{doi:10.1152/jn.00955.2002}. FiTS builds on this view by making the frequency that maximizes the magnitude response a learnable target in the FS module. It further introduces the TS module to modulate group delay, shaping when frequency-specific responses contribute to membrane voltage before threshold crossing.

%% file: sections/methods.tex
\section{Methods}
\label{sec:methods}

FiTS is a spiking neuron with two explicit mechanisms for temporal computation within each neuron. The Frequency Selectivity (FS) module parameterizes the nonzero target frequency that maximizes a neuron's subthreshold magnitude response. The Temporal Shaping (TS) module reshapes when frequency components contribute to pre-spike membrane voltage accumulation through group-delay modulation. We first present the FS module in Section~\ref{sec:fs_module}, and then build on it with the TS module in Section~\ref{sec:ts_module}.

\subsection{FS Module: Learnable Frequency Selectivity}
\label{sec:fs_module}

The standard LIF neuron follows first-order leaky membrane voltage dynamics:
\begin{equation}
\dot V(t) = -\mu V(t)+I(t), \quad \mu:=\frac{1}{\tau_m},
\label{eq:lif}
\end{equation}
where $\tau_m$ is the membrane time constant and $I(t)$ is the input current. In the spiking model, spikes are emitted when the voltage crosses the threshold $V_\mathrm{th}$, followed by a reset.

To introduce frequency-selective subthreshold dynamics, we use a LIF neuron with a voltage-dependent adaptation current, a form used in adaptive and resonant spiking neuron models~\cite{baronig2024advancing, huber2024scaling, IZHIKEVICH2001883}.
\begin{equation}
\dot V(t) = -\mu V(t)+I(t)-\eta a(t), \quad \dot a(t) = -\rho a(t)+\gamma V(t), \quad \rho:=\frac{1}{\tau_a},
\label{eq:adapt_lif}
\end{equation}
where $a(t)$ denotes the adaptation current, $\tau_a$ is the adaptation time constant, $\gamma\!\geq\!0$ is the coupling of voltage to adaptation, and $\eta\!\geq\!0$ is the feedback from adaptation to voltage. We use the subthreshold response of the system in~\eqref{eq:adapt_lif} to define an explicit target frequency for each neuron.

\subsubsection{Learnable Frequency Selectivity in Continuous Time}
\label{sec:fs_ct_design}

Under zero initial conditions, applying the Laplace transform to~\eqref{eq:adapt_lif} yields the subthreshold frequency response
\begin{equation}
H(j\Omega)=\frac{V(j\Omega)}{I(j\Omega)}=\frac{\rho+j\Omega}{(\mu\rho+\kappa-\Omega^2)+j(\mu+\rho)\Omega}=|H(j\Omega)|e^{j\phi(\Omega)}.
\label{eq:frequency_response}
\end{equation}

Here, $\kappa\!:=\!\eta\gamma$, $|H(j\Omega)|$ denotes the magnitude response, and $\phi(\Omega)$ denotes the phase response. When $|H(j\Omega)|$ has a unique nonzero global maximizer, we define the corresponding frequency as the \emph{target frequency} $\Omega^\star$. This response-level definition is distinct from a pole-based intrinsic frequency, as discussed in Appendix~\ref{app:target_vs_pole}.

\vspace{+0.8em}
\begin{theorem}[Exact target frequency parameterization in continuous time]
\label{thm:target_frequency}
Consider the subthreshold frequency response~\eqref{eq:frequency_response} with $\mu,\rho\!>\!0$. Whenever $|H(j\Omega)|$ admits a unique nonzero global maximizer over $\Omega\!>\!0$, it is given by
\begin{equation}
\Omega^\star = \sqrt{\sqrt{\kappa(2\rho^2+2\rho\mu+\kappa)}-\rho^2}.
\label{eq:omega_star_forward}
\end{equation}
Conversely, for any desired target frequency $\Omega^\star\!>\!0$, there exists a unique $\kappa^\star\!>\!0$ such that $|H(j\Omega)|$ is maximized at $\Omega^\star$. The corresponding value is
\begin{equation}
\kappa^\star = \rho(\rho+\mu)\!\left[\sqrt{1+\frac{\left(1+(\Omega^\star/\rho)^2\right)^2}{(1+\mu/\rho)^2}}-1\right].
\label{eq:kappa_inverse}
\end{equation}
\end{theorem}
\vspace{-0.8em}
A proof is provided in Appendix~\ref{app:fs_ct_proof}.

The inverse mapping is the key parameterization of the FS module. If $\kappa$ were learned directly, the realized target frequency would remain an implicit property of the learned dynamics and would have to be recovered after training. FiTS instead uses $\Omega^\star$ as the trainable coordinate and computes the corresponding $\kappa^\star$ through~\eqref{eq:kappa_inverse}. This makes the target frequency the common quantity used for initialization, optimization, and post-training interpretation.

The standard LIF dynamics in~\eqref{eq:lif} are recovered from~\eqref{eq:adapt_lif} by setting $\eta\!=\!0$, or equivalently $\kappa\!=\!0$. In that case, the magnitude response is maximized at $\Omega\!=\!0$, restricting the neuron to low-pass behavior. Thus, without the voltage-dependent adaptation current, the standard LIF neuron cannot realize a nonzero target frequency.

\subsubsection{Discrete-Time Implementation and Realized Discrete-Time Target Frequency}
\label{sec:fs_dt_realized}

Although the FS module is parameterized in continuous time for analytical tractability, practical training and inference are performed in discrete time. We therefore distinguish between the continuous-time target frequency used for parameterization and the realized discrete-time target frequency induced by the implemented update.

Let $\Delta t\!>\!0$ denote the simulation timestep and $k$ the discrete-time index. We write $V[k]$ for the carried post-reset voltage, $a[k]$ for the adaptation current, and $I[k]$ for the input at timestep $k$. Let $V_0[k+1]$ denote the pre-reset voltage of the FS module at timestep $k+1$.

We discretize~\eqref{eq:adapt_lif} using a semi-implicit Euler discretization. A comparison with a fully explicit Euler discretization is given in Appendix~\ref{app:semi_vs_explicit_euler}. The resulting update is

\vspace{-0.6em}
\begin{equation}
\begin{aligned}
V_0[k+1] &= (1-\mu\Delta t)V[k]-\eta\Delta t\,a[k]+I[k],\\
a[k+1] &= (1-\rho\Delta t)a[k]+\gamma\Delta t\,V_0[k+1].
\end{aligned}
\label{eq:fs_updates_semiimplicit}
\end{equation}

Under the subthreshold regime and zero initial conditions, applying the $z$-transform to~\eqref{eq:fs_updates_semiimplicit} and evaluating on the unit circle yields the discrete-time frequency response $H_d(e^{j\omega})$. When its magnitude response has a unique nonzero global maximizer, we denote the corresponding angular frequency by $\omega^\star_{\mathrm{DT}}$ and refer to it as the realized discrete-time target frequency.

The semi-implicit Euler update approximates the continuous-time dynamics. Thus, the realized discrete-time target $\omega^\star_{\mathrm{DT}}$ need not exactly match the continuous-time target $\Omega^\star$ in~\eqref{eq:omega_star_forward}. We compare them in common frequency units using $f^\star_{\mathrm{CT}}=\Omega^\star/(2\pi)$ and $f^\star_{\mathrm{DT}}=\omega^\star_{\mathrm{DT}}/(2\pi\Delta t)$. Consequently, we use $\Omega^\star$ for learning because it gives an explicit closed-form inverse to the adaptation strength and remains close to the realized discrete-time target when $\Delta t$ is sufficiently small. Section~\ref{sec:disc_ctdt} and Appendix~\ref{app:fs_ct_dt_map} analyze this gap.

\subsection{TS Module: Learnable Temporal Shaping}
\label{sec:ts_module}

While the FS module determines which temporal frequencies are emphasized, spike generation also depends critically on when frequency components contribute to membrane voltage accumulation. The phase response and corresponding group delay of $H_d(e^{j\omega})$ are
\begin{equation}
\phi_\mathrm{FS}(\omega):=\arg H_d(e^{j\omega}), \qquad \tau_\mathrm{FS}(\omega):=-\frac{d\phi_\mathrm{FS}(\omega)}{d\omega},
\label{eq:fs_phase_group_delay}
\end{equation}
where $\tau_\mathrm{FS}(\omega)$ is measured in samples and indicates how much frequency components near $\omega$ are delayed by the FS module.

Once the FS module parameters are fixed, the induced delay pattern is fixed as well. Thus, the neuron cannot tune frequency selectivity and threshold-crossing timing separately. The TS module introduces a complementary mechanism for temporal shaping.

\subsubsection{All-Pass Filter Cascade for Group-Delay Shaping}
\label{sec:ts_allpass_control}
To reshape the group-delay pattern induced by the FS module, we introduce an auxiliary discrete-time all-pass cascade. All-pass filters preserve magnitude while modifying phase, making them natural building blocks for temporal shaping. Specifically, we use an $M$-stage cascade of first-order all-pass filters:
\begin{equation}
A_m(z)=\frac{z^{-1}-\beta_m}{1-\beta_m z^{-1}}, \qquad A^{(M)}(z)=\prod_{m=1}^{M}A_m(z), \qquad |\beta_m|<1
\label{eq:ap_cascade}
\end{equation}
where each stage has unit magnitude on the unit circle, and thus so does the full cascade.

Applying the inverse $z$-transform to $A_m(z)$ gives the recurrence

\vspace{-0.6em}
\begin{equation}
V_m[k+1]=\beta_m\bigl(V_m[k]-V_{m-1}[k+1]\bigr)+V_{m-1}[k], \qquad m=1,\dots,M,
\label{eq:ap_recurrence}
\end{equation}
 
where $V_{m-1}$ and $V_m$ denote the input and output voltages of the $m$-th stage, respectively. Applying this recurrence stage by stage yields the cascade, with $V_0[k+1]$ as its input.

The total group delay of the cascade is the sum of the stage-wise delays. For $|\beta_m|\!<\!1$, each stage contributes a strictly positive group delay for all $\omega$. Consequently, pure cascade composition only augments the group delay induced by the FS module, rather than reshaping it more generally.

\subsubsection{Group-Delay Modulation via \texorpdfstring{$\lambda$}{lambda}-Mixing}
\label{sec:ts_lambda_mixing}

To enable richer group-delay modulation beyond pure all-pass cascade composition, the TS module recursively mixes the original FS module pathway with the outputs of the all-pass stages. With $\widetilde V_0[k+1]\!\equiv\!V_0[k+1]$, we define
\begin{equation}
\widetilde V_m[k+1]=(1-\lambda_m)\widetilde V_{m-1}[k+1]+\lambda_m V_m[k+1], \qquad m=1,\dots,M,
\label{eq:ts_nested_mix}
\end{equation}
where $\lambda_m\!\in\![0,1]$ controls the contribution of the $m$-th all-pass stage. The final mixed output $\widetilde V_M[k+1]$ serves as the effective pre-reset voltage, as illustrated in Figure~\ref{fig:ts_module}.

\begin{figure}[t]
    \centering
    \includegraphics[width=1.00\linewidth]{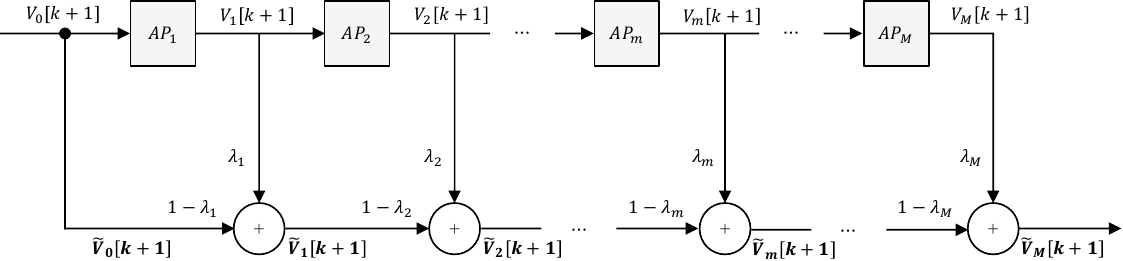}
    \caption{\textbf{TS module.} The FS output $V_0$ passes through an $M$-stage all-pass cascade and is recursively mixed with stage outputs using learnable $\lambda_m$ parameters. Each $\mathrm{AP}_m$ block denotes a first-order all-pass filter.}
    \label{fig:ts_module}
\end{figure}

The update in~\eqref{eq:ts_nested_mix} forms a convex mixture of the direct pathway and all-pass stage outputs, so the induced response remains magnitude-bounded and does not introduce arbitrary amplification. However, because the phase of a mixture is not a mixture of phases, the induced group delay need not remain a positive sum of the constituent all-pass delays. This allows $\lambda$-mixing to realize group-delay behavior beyond pure all-pass cascade composition.

\vspace{+0.6em}
\begin{proposition}[Negative group-delay shift under $\lambda$-mixing]
\label{prop:neg_delay}
In the single-stage case, $\lambda$-mixing can induce a negative group-delay shift, which is impossible under pure all-pass cascade composition.
\end{proposition}

The exact condition and proof are provided in Appendix~\ref{app:ts_negative_delay_proof}. Since every $M$-stage TS module contains the single-stage case by setting $\lambda_2\!=\!\cdots\!=\!\lambda_M\!=\!0$, this possibility extends to any $M\!\geq\!1$.

More generally, the TS module recovers the FS-only model as a special case while enabling learnable group-delay shaping through the parameters $\{\beta_m,\lambda_m\}_{m=1}^{M}$. With only $2M$ additional neuron-wise parameters, it provides a compact mechanism for temporal shaping within the neuron. To satisfy the required constraints during optimization, we parameterize $\beta_m$ with a hyperbolic tangent and use sigmoid-constrained mixing coordinates. The full discrete-time update is summarized as Algorithm~\ref{alg:FiTS_update} in Appendix~\ref{app:algorithm}.

%% file: sections/experiments.tex
\section{Experiments}

\subsection{Datasets}
\label{sec:datasets}

We evaluate FiTS on three auditory classification datasets: two spike-based benchmarks, Spiking Heidelberg Digits (SHD) and Spiking Speech Commands (SSC), and one non-spiking benchmark, Google Speech Commands (GSC).

SHD and SSC were introduced for consistent quantitative comparisons of spiking neural networks and share a common audio-to-spike conversion pipeline~\cite{cramer2020heidelberg}. SHD is a 20-class spiking speech benchmark derived from the Heidelberg Digits corpus, with spoken digits in English and German. Unlike SSC and GSC, SHD has no official validation split. We therefore report two SHD protocols in Table~\ref{tab:main_shd_ssc}. One follows the commonly used SHD protocol that reports the best test accuracy over training only for comparability with prior work. The other holds out 20\% of the training set for validation and reports test accuracy at the best validation epoch, following the validation-based SHD reporting used in SE-adLIF~\cite{baronig2024advancing} and DelRec~\cite{queant2025delrec}.

SSC is a 35-class spiking keyword-spotting benchmark derived from Google Speech Commands v0.02, with 75,466, 9,981, and 20,382 examples in the train, validation, and test splits, respectively. We also evaluate FiTS on Google Speech Commands (GSC) v0.02~\cite{warden2018speech}, a non-spiking keyword-spotting benchmark built from one-second spoken commands. For GSC, we use conventional acoustic features rather than spike-based inputs.

\input{tables/tab_main_shd_ssc}

\subsection{Experimental Setup}
\label{sec:exp_setup}

\paragraph{Input representation.}
For SHD and SSC, we follow the standard preprocessing used in recent SNN benchmarks~\cite{baronig2024advancing, deckers2024co, fabre2025structured, hammouamri2024learning, huber2024scaling, queant2025delrec}. We spatially bin the 700 afferent channels in groups of five, yielding 140 input features. Each 1.0\,s sample is discretized into $T\!=\!250$ time bins with $\Delta t\!=\!4$\,ms. For GSC, we use 40-band log-Mel spectrograms with $T\!=\!100$ acoustic frames per utterance.

\paragraph{Model configuration.}
We evaluate FiTS in a feedforward fully connected architecture with two hidden layers and a linear classifier, without explicit recurrent connections or network-level delay mechanisms, so that its contribution can be isolated at the neuron level. We use a fixed firing threshold $V_{\mathrm{th}}\!=\!1.0$ for all benchmarks and $\eta\!=\!\gamma$. We report results for multiple hidden widths and TS cascade orders $M\!\in\!\{0,1,2\}$, where $M\!=\!0$ denotes the FS-only model. For each layer, target frequencies are initialized on a log scale.

\paragraph{Training details.}
We train all FiTS and LIF-based ablation models for 100 epochs with cross-entropy loss, Adam~\cite{kingma20153rd}, cosine annealing~\cite{loshchilov2017sgdr}, a piecewise-linear triangular surrogate gradient~\cite{NEURIPS2018_c203d8a1, doi:10.1073/pnas.1604850113}, and subtractive soft reset. We use a batch size of 128 and do not use weight decay, normalization layers, learnable thresholds, auxiliary losses, or data augmentation. For our experiments, we report the mean and standard deviation across 5 random seeds. Baseline results not marked as ours are taken from the corresponding original papers. All our experiments use a single NVIDIA RTX A5000 GPU. Appendix~\ref{app:setup} reports the remaining neuron constants, selected hyperparameters, and training times.

\subsection{Main Results}
\input{tables/tab_main_gsc}

Tables~\ref{tab:main_shd_ssc} and~\ref{tab:main_gsc} summarize the main results. Overall, FiTS performs strongly in the constrained feedforward setting without recurrence or network-level delays. On SHD, FiTS achieves $95.31 \pm 0.21$\% under the best-test protocol and $94.38 \pm 0.12$\% under the 20\% validation-split protocol. On SSC and GSC, it achieves $78.23 \pm 0.16$\% and $94.48 \pm 0.12$\%, respectively.

On SHD, FiTS is competitive under both reporting protocols. Under the best-test protocol, it remains competitive with recurrent and delay-based baselines. Under the 20\% validation-split protocol, FiTS obtains the highest mean accuracy among the validation-based SHD reports in Table~\ref{tab:main_shd_ssc}. This result is obtained with fewer parameters and without recurrence or network-level delays, indicating that neuron-level frequency selectivity and temporal shaping are effective on SHD.

On SSC, FiTS remains competitive under the same feedforward setting without recurrence or network-level delays. It improves over plain LIF, TC-LIF, and cAdLIF, while remaining below stronger feedforward baselines such as DH-SFNN and C-SiLIF. These results support the usefulness of FiTS's neuron-level temporal parameterization on SSC, although stronger neuron designs, including learnable timestep parameterizations, remain important for closing the gap.

On GSC, FiTS improves over the plain LIF baseline from $80.02 \pm 0.16$\% to $94.48 \pm 0.12$\% under the same feedforward setting without recurrence or network-level delays, while remaining close to stronger adaptive baselines. This result shows that FiTS can also be applied to conventional acoustic features, not only spike-based auditory inputs.

%% file: tables/tab_main_shd_ssc.tex
\begin{table*}[t]
\centering
\footnotesize
\setlength{\tabcolsep}{10pt}
\renewcommand{\arraystretch}{1.08}
\begin{threeparttable}
\caption{\textbf{Performance comparison on spiking auditory benchmarks.}
Results are reported on SHD and SSC. We indicate whether each model uses explicit recurrent connections (Rec.) or explicit network-level delay mechanisms (Del.), together with temporal resolution (Res.), number of parameters, and top-1 accuracy. Results not marked as ours are taken from the corresponding original papers.}
\label{tab:main_shd_ssc}
\begin{tabular}{c l c c c c c}
\toprule
\textbf{Dataset} & \textbf{Method} & \textbf{Rec.} & \textbf{Del.} & \textbf{Res.} & \textbf{\# Params} & \textbf{Acc.} \\
\midrule
\multirow{10}{*}[-1.5ex]{SHD}
& RadLIF~\cite{10.3389/fnins.2022.865897}            & \cmark & \xmark & 10\,ms   & 3.9\,M   & 94.62\% \\
& S5-RF~\cite{huber2024scaling}                      & \cmark & \xmark & 4\,ms    & 0.2\,M   & 91.86\% \\
& SE-adLIF~\cite{baronig2024advancing}               & \cmark & \xmark & 4\,ms    & 0.45\,M  & 95.81 $\pm$ 0.56\% \\
& DCLS-Delays (3L-2KC)~\cite{hammouamri2024learning} & \xmark & \cmark & 10\,ms   & 0.2\,M   & 95.07 $\pm$ 0.24\% \\
& DelRec (rec. delays only)~\cite{queant2025delrec}  & \cmark & \cmark & 10\,ms   & 0.17\,M  & 93.39 $\pm$ 0.45\% \\
\cmidrule(lr){2-7}
& Plain LIF (ours)                                   & \xmark & \xmark & 4\,ms    & 0.04\,M  & 81.74 $\pm$ 0.76\% \\
& TC-LIF~\cite{zhang2024tc}                          & \xmark & \xmark & 5.6\,ms  & 0.11\,M  & 83.08\% \\
& cAdLIF~\cite{deckers2024co}                        & \xmark & \xmark & 10\,ms   & 0.04\,M  & 94.19\% \\
& DH-SFNN~\cite{zheng2024dendritic}                             & \xmark & \xmark & 1\,ms    & 0.05\,M  & 92.1\% \\
& C-SiLIF~\cite{fabre2025structured}                 & \xmark & \xmark & 4\,ms    & 0.35\,M  & 89.30 $\pm$ 1.20\% \\
& \cellcolor{gray!15}\textbf{FiTS (ours)}            & \cellcolor{gray!15}\xmark & \cellcolor{gray!15}\xmark & \cellcolor{gray!15}\textbf{4\,ms} & \cellcolor{gray!15}\textbf{0.04\,M} & \cellcolor{gray!15}\textbf{95.31 $\pm$ 0.21\%} \\
\midrule
\multirow{4}{*}[-0.3ex]{SHD$^\ast$}
& SE-adLIF~\cite{baronig2024advancing}               & \cmark & \xmark & 4\,ms    & 0.45\,M  & 93.79 $\pm$ 0.76\% \\
& DelRec (rec. delays only)~\cite{queant2025delrec}  & \cmark & \cmark & 10\,ms   & 0.24\,M  & 93.73 $\pm$ 0.69\% \\
\cmidrule(lr){2-7}
& Plain LIF (ours)                                   & \xmark & \xmark & 4\,ms    & 0.11\,M  & 79.34 $\pm$ 0.67\% \\
& \cellcolor{gray!15}\textbf{FiTS (ours)}            & \cellcolor{gray!15}\xmark & \cellcolor{gray!15}\xmark & \cellcolor{gray!15}\textbf{4\,ms} & \cellcolor{gray!15}\textbf{0.11\,M} & \cellcolor{gray!15}\textbf{94.38 $\pm$ 0.12\%} \\
\midrule
\multirow{11}{*}[-1.5ex]{SSC}
& RadLIF~\cite{10.3389/fnins.2022.865897}            & \cmark & \xmark & 10\,ms   & 3.9\,M   & 77.4\% \\
& S5-RF~\cite{huber2024scaling}                      & \cmark & \xmark & 4\,ms    & 1.75\,M  & 78.8\% \\
& SE-adLIF~\cite{baronig2024advancing}               & \cmark & \xmark & 4\,ms    & 1.6\,M   & 80.44 $\pm$ 0.26\% \\
& DCLS-Delays (3L-2KC)~\cite{hammouamri2024learning} & \xmark & \cmark & 10\,ms   & 2.5\,M   & 80.69 $\pm$ 0.21\% \\
& DelRec (rec. delays only)~\cite{queant2025delrec}  & \cmark & \cmark & 5.6\,ms  & 0.37\,M  & 82.58 $\pm$ 0.08\% \\
\cmidrule(lr){2-7}
& Plain LIF (ours)                                   & \xmark & \xmark & 4\,ms    & 0.35\,M  & 69.17 $\pm$ 0.18\% \\
& TC-LIF~\cite{zhang2024tc}                          & \xmark & \xmark & 5.6\,ms  & 0.11\,M  & 63.46\% \\
& cAdLIF~\cite{deckers2024co}                        & \xmark & \xmark & 10\,ms   & 0.35\,M  & 77.5\% \\
& DH-SFNN~\cite{zheng2024dendritic}                             & \xmark & \xmark & 1\,ms    & 0.27\,M  & 81.03\% \\
& C-SiLIF$^\dagger$~\cite{fabre2025structured}                 & \xmark & \xmark & 4\,ms    & 0.35\,M  & 81.59 $\pm$ 0.31\% \\
& \cellcolor{gray!15}\textbf{FiTS (ours)}            & \cellcolor{gray!15}\xmark & \cellcolor{gray!15}\xmark & \cellcolor{gray!15}\textbf{4\,ms} & \cellcolor{gray!15}\textbf{0.35\,M} & \cellcolor{gray!15}\textbf{78.23 $\pm$ 0.16\%} \\
\bottomrule
\end{tabular}
\begin{tablenotes}[flushleft]
\scriptsize
\item[$\ast$] SHD$^\ast$ uses a 20\% hold-out validation split from the SHD training set for model selection. Test accuracy is reported at the best validation epoch.
\item[$\dagger$] Uses a learnable neuron-wise discretization timestep instead of the standard shared timestep tied to the input-bin resolution.
\end{tablenotes}
\end{threeparttable}
\end{table*}

%% file: tables/tab_main_gsc.tex
\begin{table*}[t]
\centering
\footnotesize
\setlength{\tabcolsep}{10.5pt}
\renewcommand{\arraystretch}{1.08}
\begin{threeparttable}
\caption{\textbf{Performance comparison on GSC with non-spiking acoustic inputs.}
We report whether each model uses explicit recurrent connections (Rec.) or explicit network-level delay mechanisms (Del.), together with input frequency resolution (\# Mel), parameter count, and top-1 accuracy. All entries use a 10\,ms frame hop with approximately 100 timesteps. Results not marked as ours are taken from the corresponding original papers.}
\label{tab:main_gsc}
\begin{tabular}{c l c c c c c}
\toprule
\textbf{Dataset} & \textbf{Method} & \textbf{Rec.} & \textbf{Del.} & \textbf{\# Mel} & \textbf{\# Params} & \textbf{Acc.} \\
\midrule
\multirow{6}{*}[-1.5ex]{GSC}
& RadLIF~\cite{10.3389/fnins.2022.865897}               & \cmark & \xmark & 40  & 0.83\,M & 94.51\% \\
& DCLS-Delays (3L-2KC)~\cite{hammouamri2024learning}    & \xmark & \cmark & 140 & 2.5\,M  & 95.35 $\pm$ 0.04\% \\
\cmidrule(lr){2-7}
& Plain LIF (ours)                                      & \xmark & \xmark & 40  & 0.30\,M & 80.02 $\pm$ 0.16\% \\
& cAdLIF~\cite{deckers2024co}                           & \xmark & \xmark & 40  & 0.30\,M & 94.67\% \\
& SiLIF$^\dagger$~\cite{fabre2025structured}            & \xmark & \xmark & 40  & 1.1\,M  & 95.49 $\pm$ 0.09\% \\
& \cellcolor{gray!15}\textbf{FiTS (ours)}               & \cellcolor{gray!15}\xmark & \cellcolor{gray!15}\xmark & \cellcolor{gray!15}\textbf{40} & \cellcolor{gray!15}\textbf{0.30\,M} & \cellcolor{gray!15}\textbf{94.48 $\pm$ 0.12\%} \\
\bottomrule
\end{tabular}
\begin{tablenotes}[flushleft]
\scriptsize
\item[$\dagger$] Uses a learnable neuron-wise discretization timestep instead of the standard shared timestep tied to the input-bin resolution.
\end{tablenotes}
\end{threeparttable}
\end{table*}

%% file: sections/discussions.tex
\section{Discussion}
\label{sec:discussion}

\subsection{FiTS learns layer-dependent frequency and timing organization}
\label{sec:disc_learned_org}

\begin{figure}[!b]
    \centering
    \resizebox{\linewidth}{!}{%
        \begin{tabular}{@{}c@{\hspace{0.02\linewidth}}c@{}}
            \includegraphics{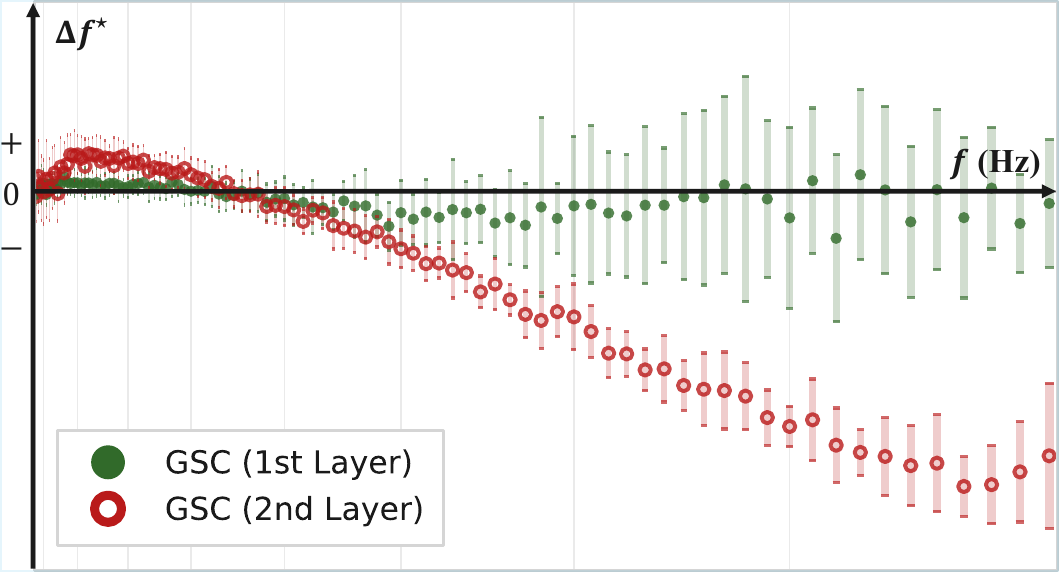} &
            \includegraphics{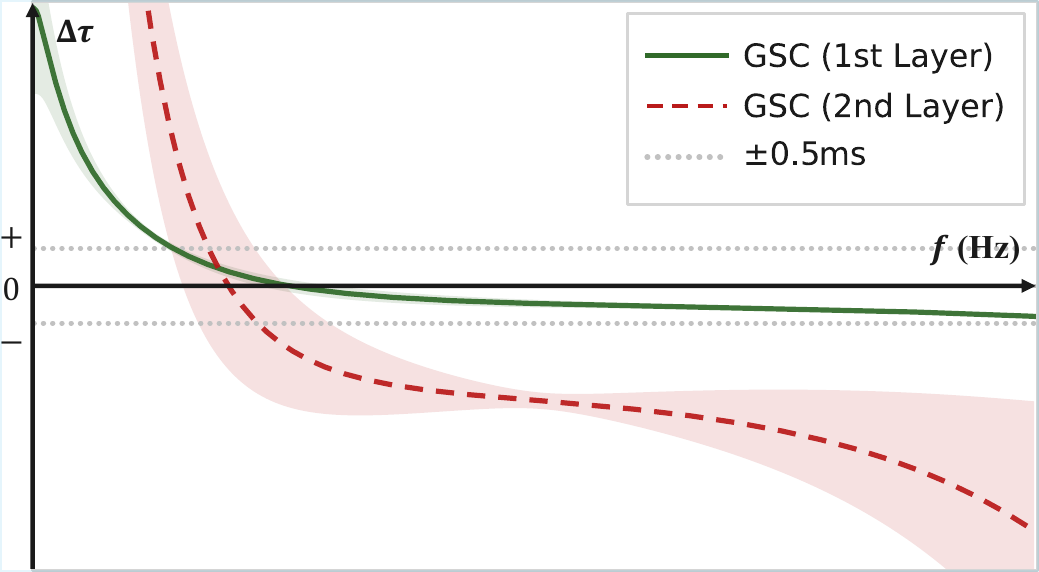}
        \end{tabular}%
    }
    \caption{\textbf{Layer-dependent organization learned by FiTS on GSC.} Left: target frequency shifts learned by the FS module across random seeds. Right: group-delay shifts learned by the TS module across random seeds.}
    \label{fig:disc_learned_org}
\end{figure}

Figure~\ref{fig:disc_learned_org} summarizes learned shifts in target frequency $\Delta f^\star$ and group delay $\Delta\tau$ across 20 random seeds. In both cases, the learned parameters exhibit clear layer-dependent organization after training.

The left panel of Figure~\ref{fig:disc_learned_org} shows how target frequencies shift under the FS module. Both layers are initialized from the same log-scale distribution of target frequencies, but their target frequencies reorganize differently after training. Across random seeds, the second layer shifts more strongly toward lower target frequencies than the first, showing that the target frequencies do not simply remain near initialization. Appendix~\ref{app:target_freq_perturb} probes these targets with reset and shuffle perturbations.

The right panel of Figure~\ref{fig:disc_learned_org} shows the learned group-delay shift induced by the TS module. In both layers, the learned delay shift is frequency-dependent and sign-indefinite, but its magnitude and shape differ across layers. Across random seeds, the second layer learns stronger and more structured group-delay shifts than the first. This suggests that the learned timing parameters summarize layer-dependent temporal organization at the parameter level. The learned distributions of $\beta$ and $\lambda$ are provided in Appendix~\ref{app:beta_lambda_dist}.

\subsection{Ablation highlights complementary contributions of FS and TS}
\label{sec:disc_ablation}

\input{tables/tab_abl_ssc}

Table~\ref{tab:ssc_hidden_size} shows that the largest gain comes from log-scale initialization of nonzero response-level target frequencies, enabled by the inverse map in Theorem~\ref{thm:target_frequency}. The adaptation-only baseline with zero-initialized adaptation shows no consistent improvement over plain LIF, whereas frozen target frequencies initialized on a log scale using~\eqref{eq:kappa_inverse} yield a significant improvement. Making these frequencies learnable through the FS module further improves accuracy at every tested hidden width. Thus, \eqref{eq:kappa_inverse} supports both learning and frequency-domain initialization by allowing the initial target frequency range to be set directly based on the simulation timestep.

On SSC, the TS module provides a consistent additional improvement over FS-only. For both $M\!=\!1$ and $M\!=\!2$, adding TS improves accuracy at every tested hidden width. Since the TS module adds only $2M$ learnable parameters per neuron on top of the FS module, the gain is not driven by a large increase in learnable parameters. Appendix~\ref{app:ssc_cascade_order_sweep} extends the SSC sweep to $M\!=\!3,4,5$, where larger cascade orders remain above FS-only without yielding monotonic gains.

\input{tables/tab_theoretical_energy}

We next characterize the accuracy--computation trade-off relative to plain LIF. We compute theoretical event-driven energy from standard AC/MAC counts using the 45\,nm CMOS costs of Horowitz~\cite{6757323}. The estimates use $E_{\mathrm{AC}}\!=\!0.9\,\mathrm{pJ}$, $E_{\mathrm{MAC}}\!=\!4.6\,\mathrm{pJ}$, and mean firing rates measured on the SSC validation split. They should be interpreted as theoretical cost estimates rather than measured hardware energy. Table~\ref{tab:ssc_energy_breakdown_w512} shows that FiTS reduces the spike-driven linear-layer cost relative to plain LIF by lowering firing rates, but shifts computation into neuron-internal updates. At hidden width 512, FS+TS ($M\!=\!2$) reduces the layer cost
from $7.3$ to $5.8\,\mu\mathrm{J}$, while increasing the
neuron-internal cost from $1.6$ to $14.5\,\mu\mathrm{J}$. The total estimated cost increases from $8.9$ to $20.3\,\mu\mathrm{J}$. This should be interpreted as an explicit accuracy--computation trade-off. FS provides a moderate-cost improvement over plain LIF, whereas TS trades additional neuron-internal computation for extra timing capacity.

\input{tables/tab_memory_energy}

We further estimate theoretical inference energy following the framework in~\cite{lemaire2022analytical} to account for arithmetic operations together with memory accesses and addressing. Table~\ref{tab:ssc_memory_energy_w512} reports the estimated energy per sample for the same SSC configuration. We assume 32-bit storage with memory costs determined by SRAM size and access FiTS states and coefficients at each timestep. Under these assumptions, FS and FS+TS require between approximately 81\% and 85\% of the energy of plain LIF. Memory accesses dominate the estimate and reduced spike-driven accesses to weights and membrane states offset the additional neuron-internal costs. The operation counts and arithmetic-cost coefficients differ from those used in Table~\ref{tab:ssc_energy_breakdown_w512}. The relative advantage depends on the assumed memory costs and can disappear for FS+TS ($M\!=\!2$) when the energy cost of accessing weights is reduced. These results are theoretical estimates rather than measured hardware energy.

\subsection{CT-to-DT fidelity of the FS parameterization}
\label{sec:disc_ctdt}

\begin{wrapfigure}{r}{0.38\linewidth}
\vspace{-12pt}
\centering
\captionsetup{justification=centering}
\caption{\textbf{CT vs. DT frequency.}}
\includegraphics[width=\linewidth]{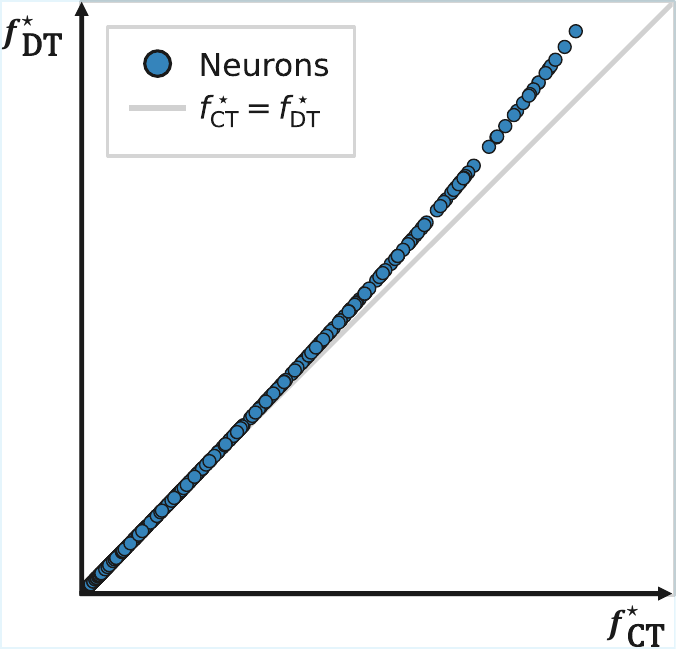}
\label{fig:disc_ctdt_fidelity}
\vspace{-25pt}
\end{wrapfigure}

The FS module is parameterized by a continuous-time target frequency, but training and inference are carried out in discrete time through semi-implicit Euler discretization. A natural question is therefore whether the learned continuous-time target frequency remains a meaningful parameter in the implemented discrete-time neuron. To measure this, we sweep the input frequency for each neuron in a model trained on SHD and identify the realized discrete-time frequency that maximizes its magnitude response. Figure~\ref{fig:disc_ctdt_fidelity} shows the relation between the learned continuous-time target frequency and the realized discrete-time target frequency across neurons. In most cases the learned continuous-time target frequency closely matches the realized discrete-time target frequency, but the gap increases toward the high-frequency regime.

The gap observed in Figure~\ref{fig:disc_ctdt_fidelity} can be understood as a consequence of using semi-implicit Euler discretization rather than an exact discretization. Appendix~\ref{app:fs_ct_dt_map} accounts for this gap by deriving closed-form stationary candidates for the implemented discrete-time response. We denote this closed-form discrete-time estimate by $\widehat f^\star_\mathrm{DT}$.

As shown in Table~\ref{tab:ctdt_alignment}, $\widehat f^\star_\mathrm{DT}$ reduces the mean absolute error from $0.7605$\,Hz to $0.0019$\,Hz and the maximum error from $7.4020$\,Hz to $0.0038$\,Hz relative to $f^\star_\mathrm{DT}$. Thus, $f^\star_\mathrm{CT}$ remains practical for learning, while $\widehat f^\star_\mathrm{DT}$ provides a near-exact post-training interpretation of the implemented discrete-time neuron.

\input{tables/tab_CTDT}

\vspace{-1.2em}

%% file: tables/tab_abl_ssc.tex
\begin{table}[t]
\centering
\begin{minipage}{1.00\linewidth}
\centering
\footnotesize
\caption{\textbf{Factorized ablation on SSC across hidden sizes.} We compare baselines with FiTS variants initialized on a log-scale target frequency range through the inverse map. The FiTS rows separate frozen target frequencies, learnable target frequencies, and TS modules.}
\label{tab:ssc_hidden_size}
\setlength{\tabcolsep}{5pt}
\begin{tabular*}{\linewidth}{@{\extracolsep{\fill}}lcccc@{}}
\toprule
\multicolumn{1}{c}{\multirow{2}{*}[-0.5ex]{\textbf{Model}}} & \multicolumn{4}{c}{\textbf{Hidden Size}} \\
\cmidrule(lr){2-5}
& \textbf{64} & \textbf{128} & \textbf{256} & \textbf{512} \\
\midrule
\multicolumn{5}{l}{\textit{\textbf{Baselines}}} \\
\hspace{0.6em}Plain LIF
& $60.91 \pm 0.28\%$ & $65.91 \pm 0.16\%$ & $68.52 \pm 0.15\%$ & $69.17 \pm 0.18\%$ \\
\hspace{0.6em}Adapt. LIF (Frozen zero init.)
& $61.18 \pm 0.22\%$ & $65.96 \pm 0.14\%$ & $68.38 \pm 0.13\%$ & $68.88 \pm 0.13\%$ \\
\midrule
\addlinespace[3pt]
\multicolumn{5}{l}{\textit{\textbf{FiTS with log-scale target frequency init.}}} \\
\hspace{0.6em}FS (Frozen target frequency)
& $69.46 \pm 0.30\%$ & $73.75 \pm 0.15\%$ & $76.15 \pm 0.20\%$ & $77.13 \pm 0.22\%$ \\
\hspace{0.6em}FS (Learnable target frequency)
& $71.72 \pm 0.12\%$ & $74.94 \pm 0.16\%$ & $76.81 \pm 0.08\%$ & $77.78 \pm 0.15\%$ \\
\hspace{0.6em}FS + TS ($M\!=\!1$)
& $72.58 \pm 0.08\%$ & $75.69 \pm 0.12\%$ & $77.28 \pm 0.06\%$ & $78.22 \pm 0.12\%$ \\
\hspace{0.6em}FS + TS ($M\!=\!2$)
& $\mathbf{73.16 \pm 0.30\%}$ & $\mathbf{75.94 \pm 0.16\%}$ & $\mathbf{77.50 \pm 0.18\%}$ & $\mathbf{78.23 \pm 0.16\%}$ \\
\bottomrule
\end{tabular*}
\end{minipage}
\end{table}

%% file: tables/tab_theoretical_energy.tex
\begin{wraptable}{r}{0.46\linewidth}
\vspace{-12pt}
\centering
\footnotesize
\caption{\textbf{Estimated theoretical energy ($\mu\mathrm{J}$) on SSC at width 512.}
Values are computed from AC/MAC counts and firing rates.}
\label{tab:ssc_energy_breakdown_w512}
\setlength{\tabcolsep}{4.5pt}
\begin{tabular}{lccc}
\toprule
\textbf{Model} & $\mathbf{E_{\text{layer}}}$ & $\mathbf{E_{\text{neuron}}}$ & $\mathbf{E_{\text{total}}}$ \\
\midrule
Plain LIF       & 7.3 & 1.6  & 8.9 \\
Adapt. (Frozen zero init.)     & 7.3 & 5.6  & 13.0 \\
\midrule
FS (Frozen target freq.)      & 6.2 & 5.6  & 11.9 \\
FS (Learnable target freq.)    & 6.2 & 5.6  & 11.8 \\
FS+TS ($M=1$)   & 5.9 & 10.1 & 16.0 \\
FS+TS ($M=2$)   & 5.8 & 14.5 & 20.3 \\
\bottomrule
\end{tabular}
\vspace{-6pt}
\end{wraptable}

%% file: tables/tab_memory_energy.tex
\begin{wraptable}{r}{0.46\linewidth}
\vspace{-1pt}
\centering
\footnotesize
\caption{\textbf{Estimated memory-aware energy ($\mu\mathrm{J}$) on SSC at width 512.}}
\label{tab:ssc_memory_energy_w512}
\setlength{\tabcolsep}{4.5pt}
\begin{tabular*}{\linewidth}{l@{\extracolsep{\fill}}c}
\toprule
\textbf{Model} & $\mathbf{E_{\text{total}}}$ \\
\midrule
Plain LIF & 839.2 \\
Adapt. (Frozen zero init.) & 858.6 \\
\midrule
FS (Frozen target freq.) & 711.5 \\
FS (Learnable target freq.) & 708.3 \\
FS+TS ($M=1$) & 681.1 \\
FS+TS ($M=2$) & 683.7 \\
\bottomrule
\end{tabular*}
\vspace{-8pt}
\end{wraptable}

%% file: tables/tab_CTDT.tex
\begin{table}[!htbp]
\centering
\begin{minipage}{1.00\linewidth}
\centering
\caption{\textbf{CT-to-DT alignment.}
Errors are measured against the realized discrete-time target frequency.}
\label{tab:ctdt_alignment}
\footnotesize
\setlength{\tabcolsep}{6pt}
\renewcommand{\arraystretch}{1.18}
\begin{tabular*}{\linewidth}{@{}p{0.54\linewidth}@{\extracolsep{\fill}}rr@{}}
\toprule
\textbf{Comparison} & \textbf{MAE (Hz)} & \textbf{Max Err. (Hz)} \\
\midrule
CT target $f^\star_\mathrm{CT}$ vs.\ realized DT $f^\star_\mathrm{DT}$
& 0.7605 & 7.4020 \\
closed-form DT $\widehat{f}^\star_\mathrm{DT}$ vs.\ realized DT $f^\star_\mathrm{DT}$
& $\mathbf{0.0019}$ & $\mathbf{0.0038}$ \\
\bottomrule
\end{tabular*}
\end{minipage}
\end{table}

%% file: sections/conclusion.tex
\section{Conclusion}

We introduced FiTS, a spiking neuron that factorizes neuron-level temporal computation into Frequency Selectivity (FS) and Temporal Shaping (TS). The FS module parameterizes the target frequency that maximizes the subthreshold magnitude response. The TS module uses group-delay modulation to shape when frequency components contribute to membrane voltage before spike generation.

Across SHD, SSC, and GSC, FiTS consistently improves over a plain LIF baseline in feedforward networks without explicit recurrence or network-level delay mechanisms. The learned target frequencies and group-delay shifts show clear layer-dependent organization. The CT-to-DT analysis further shows that the continuous-time target frequency remains practical for learning while allowing accurate discrete-time interpretation after training. These results suggest that FiTS provides a useful neuron-level temporal parameterization with parameter-level summaries of learned frequency selectivity and group-delay modulation.

This work focuses primarily on auditory benchmarks and isolates neuron-level temporal computation from stronger network-level mechanisms. Appendix~\ref{app:longer_temporal} reports additional results on sMNIST and Dynamic Vision Sensor (DVS) Gesture 128~\cite{Amir_2017_CVPR}, but broader evaluation beyond auditory tasks and interactions with recurrence, network-level delays, and richer structural neuron models remain important next steps. Finally, the TS module introduces additional neuron-internal computation while highlighting the underexplored problem of learning when frequency-specific responses contribute to pre-spike membrane voltage accumulation. This motivates more efficient implementations and broader neuron-level timing mechanisms.

%% file: sections/appendix.tex
\renewcommand{\thefigure}{A.\arabic{figure}}
\setcounter{figure}{0} 
\renewcommand{\thetable}{A.\arabic{table}}
\setcounter{table}{0} 

\begin{center}
    \textbf{\LARGE Supplementary Material: FiTS}
\end{center}
\vspace{1em}

{\hypersetup{hidelinks}
\startcontents[sections]
\printcontents[sections]{l}{1}{\setcounter{tocdepth}{3}}
}

\clearpage

\appendix

\section{Theoretical Details}
\label{app:theory}

\subsection{Proof of Theorem~\ref{thm:target_frequency}}
\label{app:fs_ct_proof}

We prove Theorem~\ref{thm:target_frequency} by characterizing the nonzero stationary point of the magnitude response in~\eqref{eq:frequency_response}, and then solving the resulting relation for $\kappa$ to obtain the inverse parameterization in~\eqref{eq:kappa_inverse}.

Starting from~\eqref{eq:frequency_response}, the squared magnitude response is
\begin{equation}
|H(j\Omega)|^2 = \frac{\rho^2+\Omega^2}{(\mu\rho+\kappa-\Omega^2)^2+(\mu+\rho)^2\Omega^2}.
\label{eq:app_fs_mag_sq}
\end{equation}
Letting $x\!:=\!\Omega^2$ with $x\!\geq\!0$, define
\begin{equation*}
M(x):=|H(j\Omega)|^2=\frac{N(x)}{D(x)},
\end{equation*}
where $N(x)$ and $D(x)$ are the numerator and denominator of~\eqref{eq:app_fs_mag_sq}, respectively. A nonzero maximizer of $|H(j\Omega)|$ over $\Omega\!>\!0$ is equivalent to a maximizer of $M(x)$ over $x\!>\!0$. Differentiating gives
\begin{equation*}
M'(x)=\frac{N'(x)D(x)-N(x)D'(x)}{\{D(x)\}^2}.
\end{equation*}
Substituting the derivatives of $N$ and $D$ and simplifying yields
\begin{equation}
x^2+2\rho^2x-\Bigl[\kappa(2\rho^2+2\rho\mu+\kappa)-\rho^4\Bigr]=0.
\label{eq:app_fs_stationary_quad}
\end{equation}
Hence the unique nonnegative stationary point is
\begin{equation}
x^\star=-\rho^2+\sqrt{\kappa(2\rho^2+2\rho\mu+\kappa)}.
\label{eq:app_fs_xstar}
\end{equation}
It is strictly positive if and only if
\begin{equation*}
\sqrt{\kappa(2\rho^2+2\rho\mu+\kappa)}>\rho^2.
\end{equation*}

To establish uniqueness, observe that the numerator of $M'(x)$ is the negative of the quadratic polynomial in~\eqref{eq:app_fs_stationary_quad}. Since this polynomial is strictly increasing on $x\!\geq\!0$, it has at most one nonnegative root. Under the condition above, that root is $x^\star\!>\!0$. Therefore $M'(x)$ changes sign from positive to negative at $x^\star$, which shows that $x^\star$ is the unique global maximizer of $M(x)$ over $x\!>\!0$. Returning to $\Omega$, this gives exactly~\eqref{eq:omega_star_forward}.

Conversely, let a desired target frequency $\Omega^\star\!>\!0$ be given and set $x^\star\!=\!(\Omega^\star)^2$. Then~\eqref{eq:app_fs_xstar} implies
\begin{equation*}
x^\star+\rho^2=\sqrt{\kappa(2\rho^2+2\rho\mu+\kappa)}.
\end{equation*}
Squaring both sides gives
\begin{equation}
\kappa^2+2\rho(\rho+\mu)\kappa-(x^\star+\rho^2)^2=0.
\label{eq:app_fs_kappa_quad}
\end{equation}
Its unique positive solution is
\begin{equation*}
\kappa^\star = -\rho(\rho+\mu) + \sqrt{\rho^2(\rho+\mu)^2+(x^\star+\rho^2)^2}.
\end{equation*}
Substituting $x^\star\!=\!(\Omega^\star)^2$ and factoring out $\rho(\rho+\mu)$ yields
\begin{equation*}
\kappa^\star = \rho(\rho+\mu)\left[\sqrt{1+\frac{\left(1+(\Omega^\star/\rho)^2\right)^2}{(1+\mu/\rho)^2}}-1\right],
\end{equation*}
which is exactly~\eqref{eq:kappa_inverse}. Since the bracketed term is strictly positive for every $\Omega^\star\!>\!0$, we have $\kappa^\star\!>\!0$, and uniqueness follows from~\eqref{eq:app_fs_kappa_quad}.

Finally, setting $\kappa\!=\!0$ recovers the standard LIF case. Then~\eqref{eq:app_fs_xstar} gives $x^\star\!=\!-\rho^2\!<\!0$, and therefore the magnitude response is maximized only at $\Omega\!=\!0$. Thus the standard LIF neuron cannot realize a nonzero target frequency. \qed

\FloatBarrier

\subsection{Comparison of Semi-Implicit and Explicit Euler Discretization}
\label{app:semi_vs_explicit_euler}

We compare the semi-implicit Euler discretization used in Section~\ref{sec:fs_dt_realized} with the fully explicit Euler discretization of~\eqref{eq:adapt_lif}. We focus on zero-input subthreshold dynamics, since they determine whether a given target frequency initialization yields a stable discrete-time state.

With $I(t)\!\equiv\!0$, the fully explicit Euler discretization of~\eqref{eq:adapt_lif} gives
\begin{equation*}
\begin{bmatrix} V[k+1] \\ a[k+1] \end{bmatrix}
=
\begin{bmatrix} 1-\mu\Delta t & -\eta\Delta t \\ \gamma\Delta t & 1-\rho\Delta t \end{bmatrix}
\begin{bmatrix} V[k]\\ a[k] \end{bmatrix},
\end{equation*}
whereas the semi-implicit discretization used in~\eqref{eq:fs_updates_semiimplicit}, with $V[k+1]\!\equiv\!V_0[k+1]$, gives
\begin{equation*}
\begin{bmatrix} V[k+1]\\ a[k+1] \end{bmatrix}
=
\begin{bmatrix} 1-\mu\Delta t & -\eta\Delta t \\ \gamma\Delta t(1-\mu\Delta t) & 1-\rho\Delta t -\kappa\Delta t^2 \end{bmatrix} \begin{bmatrix} V[k]\\ a[k] \end{bmatrix}.
\end{equation*}

For either state-update matrix above, let
\begin{equation*}
\lambda^2 - T\lambda + D,
\end{equation*}
denote its characteristic polynomial. The second-order Jury stability criterion then requires
\begin{equation*}
1-T+D>0, \qquad 1+T+D>0, \qquad 1-D>0.
\end{equation*}

In our parameter regime, the active condition is the third inequality for the explicit discretization and the second inequality for the semi-implicit discretization. Solving these conditions for $\kappa$ yields
\begin{equation}
\kappa <
\begin{cases}
\dfrac{\mu+\rho}{\Delta t}-\mu\rho, & \text{(explicit)}\\[1.0ex]
\mu\rho+\dfrac{4}{\Delta t^2}-\dfrac{2(\mu+\rho)}{\Delta t}, & \text{(semi-implicit)}
\end{cases}
\label{eq:app_kappa_bounds}
\end{equation}

\setlength{\intextsep}{4pt}
\setlength{\columnsep}{8pt}
\begin{wrapfigure}{r}{0.38\linewidth}
\vspace{-5pt}
\captionsetup{justification=raggedright,singlelinecheck=false}
\caption{\textbf{Zero-input trajectories.}}
\centering
\includegraphics[width=\linewidth]{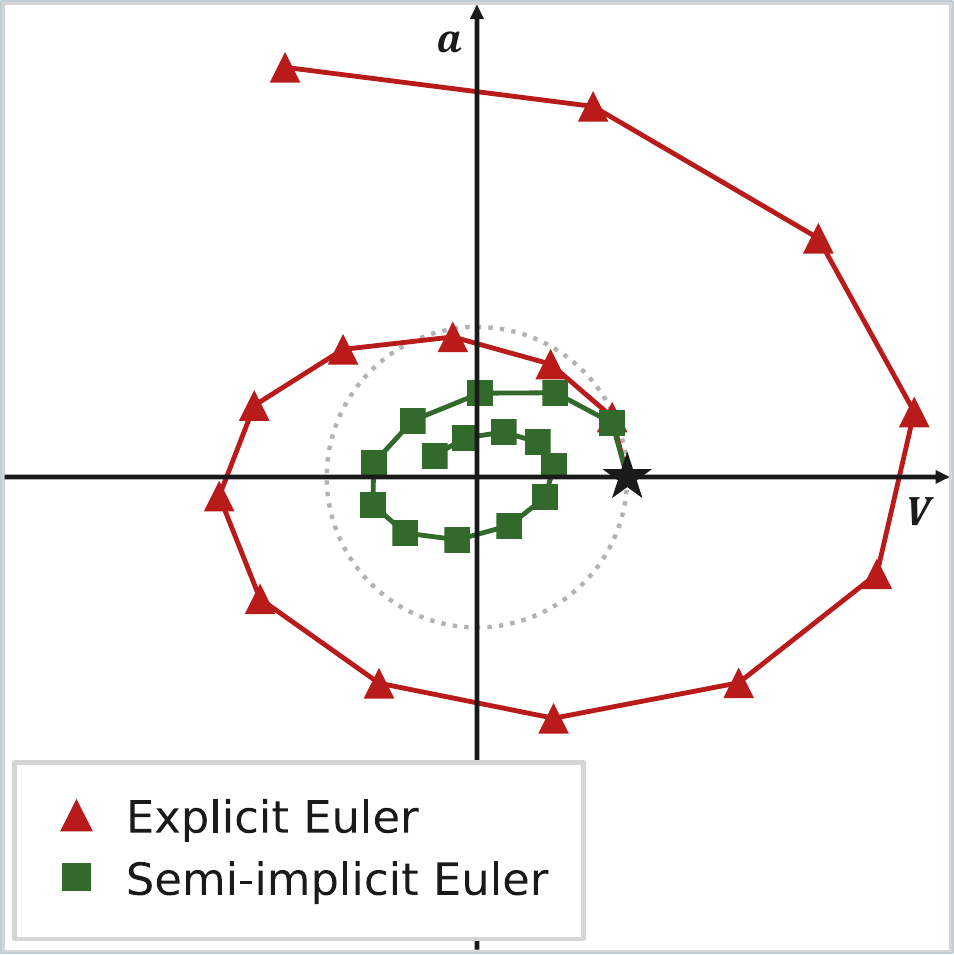}
\label{fig:appendix_stability}
\vspace{-14pt}
\end{wrapfigure}

In the SSC setting used in our experiments, $\tau_m\!=\!0.04$, $\tau_a\!=\!0.2$, and $\Delta t\!=\!0.004$. By Theorem~\ref{thm:target_frequency}, the target frequency increases monotonically with $\kappa$. Therefore, the stability bounds in~\eqref{eq:app_kappa_bounds} correspond to approximate target frequency limits of $13.8$\,Hz for the explicit discretization and $77.2$\,Hz for the semi-implicit discretization.

Figure~\ref{fig:appendix_stability} shows the different dynamical behavior induced by the two discretization methods. Under the same zero-input initialization, the explicit discretization drives the discrete state $(V,a)$ along an oscillatory divergent trajectory and across the reference unit circle in the phase plane, whereas the semi-implicit discretization shows a convergent trajectory. This is consistent with the wider stable target frequency range of the semi-implicit discretization. We therefore use the semi-implicit discretization throughout the paper.

\FloatBarrier

\subsection{Exact CT-to-DT Mapping for the FS Module}
\label{app:fs_ct_dt_map}

In this subsection, we derive the exact discrete-time transfer function induced by the semi-implicit Euler discretization~\eqref{eq:fs_updates_semiimplicit} and characterize the realized discrete-time target frequency $\omega^\star_{\mathrm{DT}}$, defined in Section~\ref{sec:fs_dt_realized} as the unique nonzero global maximizer of its magnitude response. Throughout, we interpret~\eqref{eq:fs_updates_semiimplicit} under $V[k+1]\!\equiv\!V_0[k+1]$.

\begin{theorem}[Closed-form stationary candidates for the realized discrete-time target frequency]
\label{thm:fs_ct_dt_exact}
Interpret~\eqref{eq:fs_updates_semiimplicit} under the subthreshold identification
$V[k+1]\!\equiv\!V_0[k+1]$, and define
\begin{equation*}
\bar{\mu}:=1-\mu\Delta t,\qquad
\bar{\rho}:=1-\rho\Delta t,\qquad
\bar{\kappa}:=\eta\gamma\Delta t^2.
\end{equation*}
Let $H_d(e^{j\omega})$ denote the discrete-time frequency response induced by~\eqref{eq:fs_updates_semiimplicit}. Then the stationary candidates of $|H_d(e^{j\omega})|$ over $\omega\!\in\!(0,\pi)$ are
\begin{equation}
x_\pm = \frac{1+\bar{\rho}^{\,2} \pm \sqrt{\frac{\bar{\kappa}}{\bar{\mu}}\Bigl((1-\bar{\rho}^{\,2})(1-\bar{\mu}\bar{\rho})+\bar{\kappa}\bar{\rho}\Bigr)}}{2\bar{\rho}},
\label{eq:app_fs_dt_stationary_roots_bar_theorem}
\end{equation}
with corresponding stationary frequencies
\begin{equation}
\omega_\pm=\arccos(x_\pm), \qquad \text{whenever } x_\pm\in(-1,1).
\label{eq:app_fs_dt_stationary_freqs_bar_theorem}
\end{equation}
\end{theorem}

\begin{proof}
Applying the $z$-transform to~\eqref{eq:fs_updates_semiimplicit} under zero initial conditions, together with the subthreshold identification $V[k+1]\!\equiv\!V_0[k+1]$, yields
\begin{equation*}
H_d(z) = \frac{V(z)}{I(z)} = \frac{z-\bar{\rho}}{(z-\bar{\mu})(z-\bar{\rho})+\bar{\kappa}z}.
\end{equation*}
Evaluating on the unit circle gives
\begin{equation}
H_d(e^{j\omega}) = \frac{e^{j\omega}-\bar{\rho}}{(e^{j\omega}-\bar{\mu})(e^{j\omega}-\bar{\rho})+\bar{\kappa}e^{j\omega}}, \qquad \omega\in[0,\pi].
\label{eq:app_fs_dt_transfer_uc_bar}
\end{equation}

Let $x\!:=\!\cos\omega$. Then
\begin{equation}
M(x):=\left|H_d(e^{j\omega})\right|^2 = \frac{1+\bar{\rho}^2-2\bar{\rho}x}{1+(\bar{\kappa}-\bar{\mu}-\bar{\rho})^2+\bar{\mu}^2\bar{\rho}^2-2\bar{\mu}\bar{\rho} + 2(\bar{\kappa}-\bar{\mu}-\bar{\rho})(1+\bar{\mu}\bar{\rho})x+4\bar{\mu}\bar{\rho}x^2}.
\label{eq:app_fs_dt_Mx}
\end{equation}
Since $x\!=\!\cos\omega$ is strictly decreasing on $(0,\pi)$, maximizing $|H_d(e^{j\omega})|$ over $\omega\!\in\!(0,\pi)$ is equivalent to maximizing $M(x)$ over $x\!\in\!(-1,1)$. Hence any interior maximizer must satisfy $M'(x)\!=\!0$. Differentiating~\eqref{eq:app_fs_dt_Mx} with respect to $x$ and simplifying yields
\begin{equation}
4\bar{\mu}\bar{\rho}^2x^2-4\bar{\mu}\bar{\rho}(1+\bar{\rho}^2)x+\bar{\mu}(1+\bar{\rho}^2)^2-\bar{\kappa}\Bigl((1-\bar{\rho}^2)(1-\bar{\mu}\bar{\rho})+\bar{\kappa}\bar{\rho}\Bigr)=0.
\label{eq:app_fs_dt_stationary_quad_bar}
\end{equation}

Under $0\!<\!\bar{\mu}\!<\!1$, $0\!<\!\bar{\rho}\!<\!1$, and $\bar{\kappa}\!>\!0$, the factor $(1-\bar{\rho}^2)(1-\bar{\mu}\bar{\rho})+\bar{\kappa}\bar{\rho}$ is strictly positive. Therefore the discriminant of~\eqref{eq:app_fs_dt_stationary_quad_bar} is positive, and the stationary candidates are real. Solving~\eqref{eq:app_fs_dt_stationary_quad_bar} gives~\eqref{eq:app_fs_dt_stationary_roots_bar_theorem}, and~\eqref{eq:app_fs_dt_stationary_freqs_bar_theorem} follows by setting $\omega_\pm\!=\!\arccos(x_\pm)$ whenever $x_\pm\!\in\!(-1,1)$.
\end{proof}

Theorem~\ref{thm:fs_ct_dt_exact} provides a closed-form characterization of the stationary candidates of the realized discrete-time target frequency. Given a continuous-time target frequency, the induced $\bar{\kappa}$ defines the exact discrete-time frequency response. The realized discrete-time target frequency is then obtained by selecting the valid maximizer among the stationary candidates.

\FloatBarrier

\subsection{Proof of Proposition~\ref{prop:neg_delay}}
\label{app:ts_negative_delay_proof}

We consider the single-stage case, since the multi-stage TS module contains it by setting $\lambda_2\!=\!\cdots\!=\!\lambda_M\!=\!0$. Let
\begin{equation*}
A_1(e^{j\omega}) = e^{j\phi_{A_1}(\omega)}
\end{equation*}
denote the first-stage all-pass response, and let
\begin{equation}
G_1(e^{j\omega}) := (1-\lambda_1) + \lambda_1 A_1(e^{j\omega})
\label{eq:app_ts_single_stage_mix}
\end{equation}
be the corresponding single-stage $\lambda$-mixed response.

For pure all-pass composition, each stage contributes a nonnegative all-pass delay
\begin{equation*}
\tau_{A_1}(\omega):=-\frac{d}{d\omega}\phi_{A_1}(\omega)\geq 0.
\end{equation*}
Thus pure all-pass cascade composition can only produce nonnegative stage-wise group-delay shifts. Under $\lambda$-mixing, this restriction no longer holds.

Writing
\begin{equation*}
G_1(e^{j\omega}) = \bigl((1-\lambda_1)+\lambda_1\cos\phi_{A_1}(\omega)\bigr)+j\,\lambda_1\sin\phi_{A_1}(\omega),
\end{equation*}
the group delay of $G_1(e^{j\omega})$ is obtained by differentiating its phase. A direct calculation yields
\begin{equation}
\tau_{G_1}(\omega) = \frac{\lambda_1\bigl(\lambda_1+(1-\lambda_1)\cos\phi_{A_1}(\omega)\bigr)}{\lambda_1^2+(1-\lambda_1)^2+2\lambda_1(1-\lambda_1)\cos\phi_{A_1}(\omega)}\tau_{A_1}(\omega).
\label{eq:app_ts_single_stage_delay}
\end{equation}

The denominator in~\eqref{eq:app_ts_single_stage_delay} is
\begin{equation*}
\left|(1-\lambda_1)+\lambda_1 e^{j\phi_{A_1}(\omega)}\right|^2,
\end{equation*}
and is therefore nonnegative, with equality only in the degenerate cancellation case. Since $\lambda_1\!\in\!(0,1)$ and $\tau_{A_1}(\omega)\!\geq\!0$, the sign of $\tau_{G_1}(\omega)$ is determined by
\begin{equation*}
\lambda_1+(1-\lambda_1)\cos\phi_{A_1}(\omega).
\end{equation*}
Hence
\begin{equation}
\tau_{G_1}(\omega)<0 \quad \Longleftrightarrow \quad \lambda_1+(1-\lambda_1)\cos\phi_{A_1}(\omega)<0.
\label{eq:app_ts_neg_delay_sign}
\end{equation}
Solving this inequality yields the exact condition
\begin{equation}
\lambda_1<\frac{1}{2} \qquad \text{and} \qquad \cos\phi_{A_1}(\omega)<-\frac{\lambda_1}{1-\lambda_1}.
\label{eq:app_ts_neg_delay_condition}
\end{equation}

This proves the single-stage claim. The extension to any $M\!\geq\!1$ follows immediately by setting $\lambda_2\!=\!\cdots\!=\!\lambda_M\!=\!0$.

\FloatBarrier

\subsection{Algorithmic Form of the FiTS Update}
\label{app:algorithm}

Algorithm~\ref{alg:FiTS_update} summarizes the full discrete-time update of the FiTS neuron. The update proceeds in three stages: the FS module computes the subthreshold voltage and adaptation update, the TS module applies stage-wise all-pass filtering and $\lambda$-mixing, and the final mixed response is used for spike generation and soft reset.

\begin{algorithm}[H]
  \footnotesize
  \caption{Discrete-Time Update of the FiTS Neuron}
  \label{alg:FiTS_update}
  \renewcommand{\baselinestretch}{1.05}\selectfont
  \begin{algorithmic}[1]
    \Statex \textbf{Input:} $V[k],\, a[k],\, \{V_m[k]\}_{m=0}^{M},\, I[k]$
    \Statex \textbf{Output:} $V[k+1],\, a[k+1],\, \{V_m[k+1]\}_{m=0}^{M},\, S[k+1]$

    \Statex \colorbox{gray!12}{\makebox[\dimexpr\linewidth-2\fboxsep\relax][l]{\textbf{FS module}}}
    \State $V_0[k+1] \gets (1-\mu\Delta t)V[k] - \eta\Delta t\,a[k] + I[k]$
    \State $a[k+1] \gets (1-\rho\Delta t)a[k] + \gamma\Delta t\,V_0[k+1]$

    \Statex \colorbox{gray!12}{\makebox[\dimexpr\linewidth-2\fboxsep\relax][l]{\textbf{TS module}}}
    \State $\widetilde V_0[k+1] \gets V_0[k+1]$
    \For{$m=1,\dots,M$}
        \State $V_m[k+1] \gets \beta_m\bigl(V_m[k]-V_{m-1}[k+1]\bigr)+V_{m-1}[k]$
        \State $\widetilde V_m[k+1] \gets (1-\lambda_m)\widetilde V_{m-1}[k+1] + \lambda_m V_m[k+1]$
    \EndFor

    \Statex \colorbox{gray!12}{\makebox[\dimexpr\linewidth-2\fboxsep\relax][l]{\textbf{Spike Generation and Reset}}}
    \State $S[k+1] \gets \Theta(\widetilde V_M[k+1]-V_\mathrm{th})$
    \State $V[k+1] \gets \widetilde V_M[k+1]- S[k+1]V_\mathrm{th}$
  \end{algorithmic}
\end{algorithm}

\FloatBarrier

\subsection{Response-Level Target Frequency versus Intrinsic Pole Frequency}
\label{app:target_vs_pole}

This section clarifies the distinction between the FiTS target frequency and an intrinsic frequency defined from the homogeneous dynamics. Starting from the input-to-voltage response in~\eqref{eq:frequency_response}, the squared magnitude response is
\begin{equation}
|H(j\Omega)|^2 = \frac{\rho^2+\Omega^2}{(\mu\rho+\kappa-\Omega^2)^2+(\mu+\rho)^2\Omega^2}.
\label{eq:app_mag_response_sq}
\end{equation}

FiTS defines the target frequency as the nonzero maximizer of this input-to-voltage magnitude response. By~\eqref{eq:omega_star_forward}, this response-level target satisfies
\begin{equation}
(\Omega^\star)^2 =\sqrt{\kappa(2\rho^2+2\rho\mu+\kappa)}-\rho^2.
\label{eq:app_omega_star_sq}
\end{equation}

By contrast, an intrinsic pole frequency is determined by the homogeneous part of the same linear subsystem. The denominator of~\eqref{eq:frequency_response} gives the characteristic polynomial

\begin{equation}
s^2+(\mu+\rho)s+(\mu\rho+\kappa)=0.
\label{eq:app_char_poly}
\end{equation}
When the poles are complex, the magnitude of their imaginary part is
\begin{equation}
\Omega_{\mathrm{pole}} = \frac{1}{2}\sqrt{4(\mu\rho+\kappa)-(\mu+\rho)^2},
\label{eq:app_pole_frequency}
\end{equation}
and hence
\begin{equation}
\Omega_{\mathrm{pole}}^2 = \mu\rho+\kappa-\frac{(\mu+\rho)^2}{4}.
\label{eq:app_pole_frequency_sq}
\end{equation}

This quantity characterizes the natural mode of the homogeneous subsystem. It is not, in general, the same as the frequency at which input current most strongly drives membrane voltage.

The difference follows directly from Eqs.~\eqref{eq:app_omega_star_sq} and~\eqref{eq:app_pole_frequency_sq}. Equating the response-level target frequency and the pole-imaginary frequency would require

\begin{equation}
\sqrt{\kappa(2\rho^2+2\rho\mu+\kappa)}-\rho^2 = \mu\rho+\kappa-\frac{(\mu+\rho)^2}{4},
\label{eq:app_target_pole_constraint}
\end{equation}

which is not satisfied in general. The response-level target frequency depends on both the numerator and denominator of the input-to-voltage transfer function, whereas the imaginary part of the pole depends only on the homogeneous denominator.

This distinction also separates response-level frequency selectivity from intrinsic oscillation. A nonzero imaginary part of the pole exists only when
\begin{equation}
4(\mu\rho+\kappa)>(\mu+\rho)^2.
\label{eq:app_complex_pole_condition}
\end{equation}

However, this condition is not equivalent to the existence of a nonzero maximizer of $|H(j\Omega)|$. For example, with $\mu\!=\!4$, $\rho\!=\!1$, and $\kappa\!=\!2$, we have
\begin{equation*}
4(\mu\rho+\kappa)=24 < 25=(\mu+\rho)^2,
\end{equation*}
so the poles are real and $\Omega_{\mathrm{pole}}$ is not a real-valued intrinsic oscillation frequency. Nevertheless,~\eqref{eq:app_omega_star_sq} gives
\begin{equation*}
(\Omega^\star)^2=\sqrt{24}-1>0,
\end{equation*}
so the input-to-voltage response has a nonzero target frequency $\Omega^\star\!\approx\!1.97$. Thus, a neuron can exhibit a response-level frequency preference even when its homogeneous dynamics do not have an oscillatory eigenmode.

This response-level definition is relevant for spiking neurons because spike generation is determined by threshold crossing of the membrane voltage. FiTS therefore parameterizes the frequency at which input current most strongly contributes to membrane voltage, rather than only the natural frequency of the autonomous linear subsystem.

\FloatBarrier

\section{Experimental Setup and Reproducibility}
\label{app:setup}

\subsection{Additional experimental settings}
\label{app:additional_settings}

Table~\ref{tab:app_common_settings} reports the dataset-specific neuron constants used in the benchmark experiments. Shared training settings are given in the main text. For all datasets, we set the adaptation current scale ratio to $\eta/\gamma\!=\!1.0$. The TS parameters are initialized in unconstrained coordinates. We initialize the all-pass coordinate as $\hat{\beta}_m\!=\!0$, which gives $\beta_m\!=\!\tanh(\hat{\beta}_m)\!=\!0$. The implementation mixes stage outputs from the last stage toward the FS pathway using $q_m=\sigma(\hat q_m)$ and optimizes $\hat q_m$ initialized at $-3.0$, giving $q_m\approx0.047$. This is equivalent to~\eqref{eq:ts_nested_mix} under $\lambda_m=(P_m-P_{m+1})/(1-P_{m+1})$, where $P_m=\prod_{i=1}^{m}q_i$ and $P_{M+1}=0$. Thus, the TS module is initialized close to the direct FS pathway and learns the amount of all-pass mixing during training.

\input{tables/appendix/tab_settings}

For reproducibility, Table~\ref{tab:app_lr_dropout} reports the selected learning rate and dropout probability for each dataset, cascade order, and hidden width. These values correspond to the final configurations used in the benchmark experiments.

\input{tables/appendix/tab_lr_dropout}

\FloatBarrier

\subsection{Training time}
\label{app:training_time}

FiTS is evaluated in feedforward SNNs without recurrence or network-level delays. This allows layer-wise temporal processing. Once the spike sequence from the previous layer is available, the temporal update of the next layer can be evaluated over the full sequence. We use this property to implement the FiTS temporal update with a fused Triton kernel.

The fused implementation processes temporal chunks within each layer and combines the FS module update, TS module update, spike generation, soft reset, and backward pass. To isolate the implementation effect, we compare it with a PyTorch timestep loop under a setting with two hidden layers and SSC-like tensor shapes. This benchmark uses the same temporal update in both cases and changes only the implementation. The two implementations produce identical spike outputs in all tested settings.

\input{tables/appendix/tab_kernel_runtime}

Table~\ref{tab:app_kernel_runtime} shows that the fused Triton implementation reduces forward and backward runtime by $31\times$ to $135\times$ relative to the PyTorch timestep loop across the tested hidden widths and TS cascade orders. This speedup comes from evaluating the temporal update within a fused layer-wise implementation rather than changing the mathematical operations of the neuron model. These measurements are implementation-level runtime results. They are separate from the theoretical operation-level energy estimates in Table~\ref{tab:ssc_energy_breakdown_w512}, and should not be interpreted as speedups for architectures with inter-layer temporal recurrence.

Table~\ref{tab:app_training_time} reports the resulting per-run training times used in our benchmark experiments. Times are measured in minutes on a single NVIDIA RTX A5000 GPU and summarized by the median over completed runs. The measured times are similar across $M=0,1,2$ in the tested settings, indicating that the fused implementation keeps feedforward FiTS training practical when TS cascade orders are used.

\input{tables/appendix/tab_training_time}

\FloatBarrier

\section{Additional Results}
\label{app:additional_results}

\subsection{Additional Ablation}
\label{app:app_ablation}

\paragraph{Factorized ablations on GSC and SHD$^\ast$.}
Tables~\ref{tab:gsc_hidden_size} and~\ref{tab:shdstar_hidden_size} provide the full ablations for GSC and SHD$^\ast$. They follow the same factorized comparison as the SSC ablation in the main text. Log-scale initialization of target frequencies gives a large improvement over plain LIF on both benchmarks. The learnable FS module and TS module further refine this baseline, although the TS effect depends on the dataset, hidden width, and cascade order.

\input{tables/appendix/tab_abl_gsc}
\input{tables/appendix/tab_abl_shd}

\paragraph{Larger TS cascade orders on SSC.}
\label{app:ssc_cascade_order_sweep}

To test whether the gains from the TS module are explained by a larger cascade order alone, we extend the SSC ablation to $M\!=\!3,4,5$. Table~\ref{tab:app_ssc_m_sweep} reports the extended sweep at the selected hidden widths. This experiment varies only the TS cascade order while keeping the rest of the architecture and training protocol unchanged.

The TS module consistently improves over FS-only across all tested cascade orders. However, the improvement is not monotonic in $M$. Thus, the cascade order is better viewed as a hyperparameter controlling temporal shaping capacity rather than as a simple scaling axis. More systematic order selection and broader within-neuron mechanisms for shaping pre-spike membrane voltage accumulation remain useful directions for future work.

\input{tables/appendix/tab_abl_ssc_extend}

\FloatBarrier

\subsection{Sensitivity to Target-Frequency Initialization}
\label{app:frequency_initialization}

We compare target-frequency initialization schemes on SHD$^\ast$ using FS-only networks ($M\!=\!0$) with hidden width 256. Target frequencies are evenly spaced on a log scale or a linear scale over $[1,50]$\,Hz. We also evaluate constant initializations that assign all neurons a target frequency of $\sqrt{50}$\,Hz or $25$\,Hz. The value $\sqrt{50}$\,Hz is the geometric mean of the interval endpoints. For each initialization, we evaluate both frozen and learnable target frequencies. All configurations use 100 training epochs with a learning rate of $2\times10^{-3}$ and dropout of $0.3$. The learning rate and dropout were selected for log-scale initialization and kept unchanged across all initialization schemes.

\input{tables/appendix/tab_freq_initialization}

Table~\ref{tab:freq_initialization} shows that log-scale initialization yields the highest observed mean accuracy in both the frozen and learnable settings. Learning the target frequencies improves the mean accuracy for log-scale initialization and both constant initializations but decreases it for linear-scale initialization. Constant initialization at $25$\,Hz performs substantially worse than the other initializations in both settings. These results support log-scale initialization under the evaluated training configuration and show that learning target frequencies does not remove sensitivity to initialization.

\subsection{Target Frequency Perturbation}
\label{app:target_freq_perturb}

To examine whether the target frequencies learned by the FS module remain relevant to the trained computation, we perturb SSC checkpoints trained with the FS module only ($M\!=\!0$) at inference time. For each hidden width, we evaluate 5 independently trained checkpoints. We compare the original model with two perturbations. In the reset condition, we keep all feedforward weights and other parameters fixed and restore only the target frequency parameters to their initial log-scale values. In the shuffle condition, we randomly permute the target frequency parameters within each layer. This preserves the layer-wise target frequency distribution but breaks the assignment between target frequencies and individual neurons.

\input{tables/appendix/tab_target_freq_perturb}

Table~\ref{tab:app_target_freq_perturb} reports the perturbed accuracies and their changes relative to the original checkpoints. Resetting the target frequencies degrades accuracy at every hidden width. The learned target frequencies therefore remain active components of the trained computation after optimization. The degradation is larger at smaller widths, indicating that narrower networks are more sensitive to this perturbation. The shuffle condition causes a larger degradation and brings accuracy close to chance level across all widths, even though the layer-wise target frequency distribution is preserved. Thus, both the learned distribution and the assignment of target frequencies to individual neurons are coupled to the learned feedforward weights and the resulting spiking computation.

\FloatBarrier

\subsection{Learned Distributions of \texorpdfstring{$\beta$}{beta} and \texorpdfstring{$\lambda$}{lambda}}
\label{app:beta_lambda_dist}

To complement the layer-dependent organization analysis in Section~\ref{sec:disc_learned_org}, we visualize the learned TS parameters $\beta$ and $\lambda$ on GSC. Figure~\ref{fig:app_beta_lambda_dist} shows one representative trained checkpoint from the 20 random seeds analyzed in Section~\ref{sec:disc_learned_org}. This figure provides a qualitative view of the learned TS parameters rather than an additional aggregate statistic.

The learned $\beta$ values do not remain concentrated near zero after training. The two hidden layers form distinct distributions within the stable range $(-1,1)$, which indicates that the all-pass stage is used for group-delay modulation rather than remaining at the initial pure-delay setting. The learned $\lambda$ values also do not collapse to a trivial endpoint. They occupy a nontrivial subset of $(0,1)$, with the second hidden layer placing substantial mass on intermediate mixing values. These distributions indicate that the TS module uses both the direct FS pathway and the all-pass pathway after training.

These checkpoint-level distributions are consistent with the group-delay patterns reported in Figure~\ref{fig:disc_learned_org}. The aggregate group-delay analysis in Section~\ref{sec:disc_learned_org} remains the main evidence for layer-dependent timing organization, while this figure shows how the underlying TS parameters are distributed in a representative trained model.

\begin{figure}[!htbp]
\centering
\resizebox{\linewidth}{!}{%
    \begin{tabular}{@{}c@{\hspace{0.03\linewidth}}c@{}}
        \includegraphics{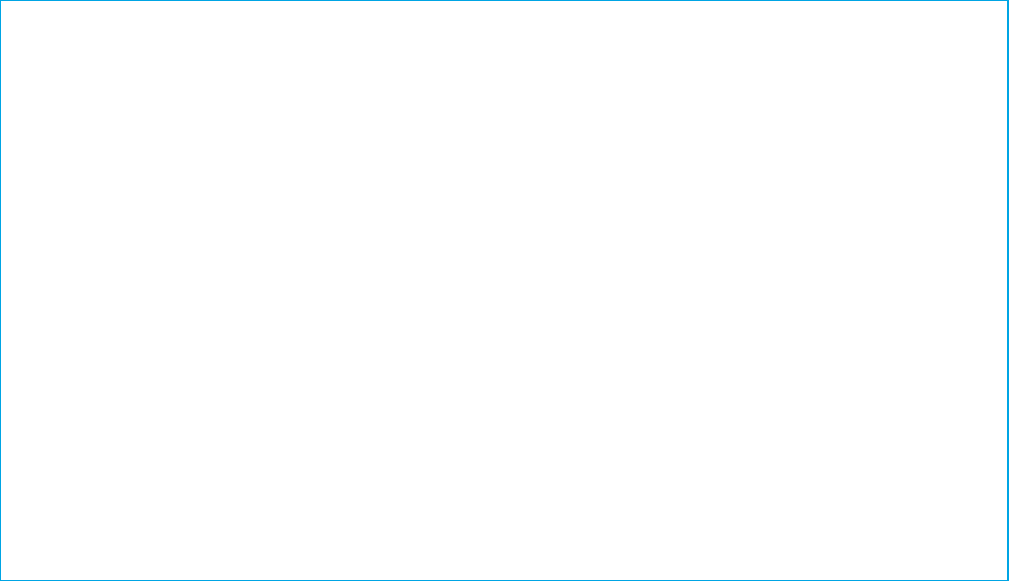} &
        \includegraphics{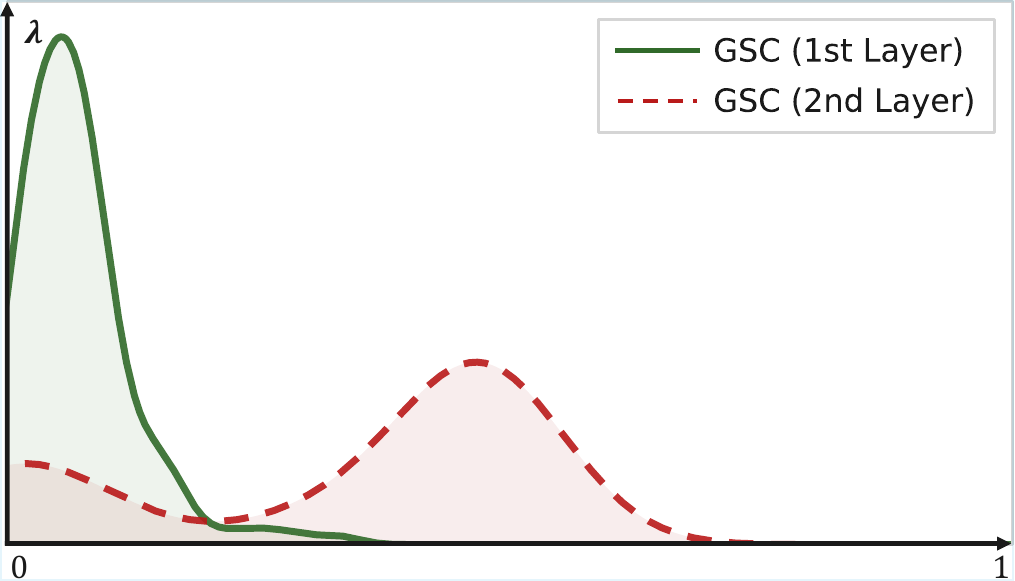}
    \end{tabular}%
}
\caption{\textbf{Representative learned distributions of TS parameters on GSC.}
The distributions are shown for one trained checkpoint selected from the 20 random seeds analyzed in Section~\ref{sec:disc_learned_org}. Left: learned all-pass parameters $\beta\!\in\!(-1,1)$. Right: learned mixing parameters $\lambda\!\in\!(0,1)$.}
\label{fig:app_beta_lambda_dist}
\end{figure}

\FloatBarrier

\subsection{Additional Temporal Benchmark Results}
\label{app:longer_temporal}

\paragraph{sMNIST.}
We include sequential MNIST (sMNIST) as an additional non-auditory evaluation. Table~\ref{tab:app_smnist} reports a representative FiTS result together with previously reported baselines. FiTS reaches $98.28\%$ accuracy with $0.07$\,M parameters using the $M\!=\!1$, width-256 model. This result is not intended to establish state-of-the-art performance on sMNIST, which primarily stresses long-range temporal credit assignment. Rather, it shows that the same neuron-level frequency selectivity and group-delay modulation can be applied outside the auditory benchmarks considered in the main text.

\input{tables/appendix/tab_smnist}

\paragraph{DVS Gesture 128.}
We also evaluate DVS Gesture 128~\cite{Amir_2017_CVPR} to test FiTS on non-auditory event-based input. The goal is not to compete with state-of-the-art DVS Gesture models, but to test whether FiTS improves over plain LIF within the same lightweight feedforward convolutional architecture. Events are converted into $T\!=\!50$ time bins with two polarity channels and resized to $64\!\times\!64$. The model has two convolutional spiking blocks, one fully connected spiking block, and a linear classifier. All variants use the same architecture and fixed training recipe, and differ only in the hidden-layer neuron model: plain LIF, FS-only ($M\!=\!0$), or FS+TS ($M\!=\!1,2$).

For convolutional layers, FiTS parameters are shared across spatial locations within each channel. Each channel learns one target frequency and, for TS variants, one $M$-stage set of $\beta$ and $\lambda$ parameters. The fully connected hidden layer uses unit-wise FiTS parameters. Since DVS Gesture 128 has no official validation split, we use the official train/test split and report final-epoch test accuracy across 5 random seeds under a fixed training schedule. With $T\!=\!50$ and $\Delta t\!=\!20$ ms, we restrict the target frequency range to $[0.5, 5]$\,Hz for stable FiTS updates.

\input{tables/appendix/tab_dvsgesture}
\vspace{+0.6em}
Table~\ref{tab:app_dvsgesture} shows that FiTS improves over plain LIF under the same lightweight architecture. FS-only improves accuracy by $2.71$ percentage points, and adding TS further improves the result to $90.35\%$. This provides an additional non-auditory event-based example where the proposed frequency and timing parameterization within each neuron is useful beyond the auditory benchmarks in the main text.

%% file: tables/appendix/tab_settings.tex
\begin{table}[!htbp]
\centering
\footnotesize
\begin{threeparttable}
\caption{\textbf{Dataset-specific neuron constants.}
The firing threshold and shared training settings are reported in the main text.}
\label{tab:app_common_settings}
\renewcommand{\arraystretch}{1.12}
\setlength{\tabcolsep}{6pt}
\begin{tabular*}{0.78\linewidth}{@{\extracolsep{\fill}}lccc@{}}
\toprule
\textbf{Setting} & \textbf{SHD} & \textbf{SSC} & \textbf{GSC} \\
\midrule
$\tau_m$ & $0.04\,\mathrm{s}$ & $0.04\,\mathrm{s}$ & $0.10\,\mathrm{s}$ \\
$\tau_a$ & $0.20\,\mathrm{s}$ & $0.20\,\mathrm{s}$ & $0.50\,\mathrm{s}$ \\
$\Delta t$ & $0.004\,\mathrm{s}$ & $0.004\,\mathrm{s}$ & $0.010\,\mathrm{s}$ \\
Target frequency range $f^\star$ & $[1,50]\,\mathrm{Hz}$ & $[1,50]\,\mathrm{Hz}$ & $[1,30]\,\mathrm{Hz}$ \\
\bottomrule
\end{tabular*}
\end{threeparttable}
\end{table}

%% file: tables/appendix/tab_lr_dropout.tex
\begin{table}[!htbp]
\centering
\scriptsize
\caption{\textbf{Selected learning rate and dropout.}
Each entry reports learning rate $lr$ and dropout probability $p$ by dataset,
cascade order, and hidden width.}
\label{tab:app_lr_dropout}
\renewcommand{\arraystretch}{1.12}
\setlength{\tabcolsep}{3.8pt}
\begin{tabular*}{\linewidth}{@{\extracolsep{\fill}}llcccc@{}}
\toprule
\textbf{Dataset} & \textbf{Order}
& $\mathbf{W=64}$ & $\mathbf{W=128}$ & $\mathbf{W=256}$ & $\mathbf{W=512}$ \\
\midrule

\multirow{3}{*}{SHD}
& $M=0$ & $lr=3\mathrm{e}{-3},\,p=0.5$ & $lr=6.5\mathrm{e}{-3},\,p=0.4$ & $lr=2\mathrm{e}{-3},\,p=0.3$ & $lr=1.5\mathrm{e}{-3},\,p=0.4$ \\
& $M=1$ & $lr=9\mathrm{e}{-3},\,p=0.2$ & $lr=7.5\mathrm{e}{-3},\,p=0.5$ & $lr=7.5\mathrm{e}{-4},\,p=0.2$ & $lr=2\mathrm{e}{-3},\,p=0.3$ \\
& $M=2$ & $lr=5.5\mathrm{e}{-3},\,p=0.6$ & $lr=6.5\mathrm{e}{-3},\,p=0.6$ & $lr=4.5\mathrm{e}{-3},\,p=0.3$ & $lr=2\mathrm{e}{-2},\,p=0.2$ \\

\midrule

\multirow{3}{*}{SSC}
& $M=0$ & $lr=2\mathrm{e}{-3},\,p=0.3$ & $lr=1.5\mathrm{e}{-3},\,p=0.3$ & $lr=6\mathrm{e}{-4},\,p=0.45$ & $lr=6\mathrm{e}{-4},\,p=0.55$ \\
& $M=1$ & $lr=1.25\mathrm{e}{-3},\,p=0.2$ & $lr=1\mathrm{e}{-3},\,p=0.35$ & $lr=5.5\mathrm{e}{-4},\,p=0.5$ & $lr=5.5\mathrm{e}{-4},\,p=0.6$ \\
& $M=2$ & $lr=3\mathrm{e}{-3},\,p=0.3$ & $lr=1.5\mathrm{e}{-3},\,p=0.45$ & $lr=8\mathrm{e}{-4},\,p=0.55$ & $lr=4.5\mathrm{e}{-4},\,p=0.55$ \\

\midrule

\multirow{3}{*}{GSC}
& $M=0$ & $lr=4\mathrm{e}{-3},\,p=0.15$ & $lr=2.5\mathrm{e}{-3},\,p=0.20$ & $lr=2.5\mathrm{e}{-3},\,p=0.3$ & $lr=9\mathrm{e}{-4},\,p=0.45$ \\
& $M=1$ & $lr=3.5\mathrm{e}{-3},\,p=0.2$ & $lr=3\mathrm{e}{-3},\,p=0.25$ & $lr=3\mathrm{e}{-3},\,p=0.4$ & $lr=3\mathrm{e}{-3},\,p=0.65$ \\
& $M=2$ & $lr=4\mathrm{e}{-3},\,p=0.2$ & $lr=3.5\mathrm{e}{-3},\,p=0.3$ & $lr=2\mathrm{e}{-3},\,p=0.55$ & $lr=1\mathrm{e}{-3},\,p=0.45$ \\

\bottomrule
\end{tabular*}
\end{table}

%% file: tables/appendix/tab_kernel_runtime.tex
\begin{table}[!htbp]
\centering
\footnotesize
\begin{threeparttable}
\begin{minipage}{1.00\linewidth}
\caption{\textbf{Runtime of the temporal FiTS update.}
We compare a PyTorch timestep loop with the fused Triton implementation under a setting with two hidden layers and SSC-like tensor shapes. Times include forward and backward passes and are reported as median milliseconds per iteration on a single NVIDIA RTX A5000 GPU. Bold values in parentheses report speedup over the PyTorch loop.}
\label{tab:app_kernel_runtime}

\setlength{\tabcolsep}{5.5pt}
\renewcommand{\arraystretch}{1.10}
\begin{tabular*}{\linewidth}{@{\extracolsep{\fill}}cccr@{}}
\toprule
$W$ & $M$ & PyTorch (ms) & Triton (ms) \\
\midrule
\multirow{3}{*}{64}
& $M=0$ & 426.7 & $8.3$ \textbf{\boldmath$(51.16{\times})$} \\
& $M=1$ & 627.1 & $7.3$ \textbf{\boldmath$(86.35{\times})$} \\
& $M=2$ & 861.8 & $6.7$ \textbf{\boldmath$(129.04{\times})$} \\
\midrule
\multirow{3}{*}{128}
& $M=0$ & 521.8 & $7.3$ \textbf{\boldmath$(71.13{\times})$} \\
& $M=1$ & 704.2 & $12.8$ \textbf{\boldmath$(54.81{\times})$} \\
& $M=2$ & 876.6 & $10.4$ \textbf{\boldmath$(84.35{\times})$} \\
\midrule
\multirow{3}{*}{256}
& $M=0$ & 537.8 & $17.3$ \textbf{\boldmath$(31.09{\times})$} \\
& $M=1$ & 933.3 & $6.9$ \textbf{\boldmath$(134.56{\times})$} \\
& $M=2$ & 1099.9 & $20.6$ \textbf{\boldmath$(53.30{\times})$} \\
\midrule
\multirow{3}{*}{512}
& $M=0$ & 1079.4 & $33.5$ \textbf{\boldmath$(32.20{\times})$} \\
& $M=1$ & 1385.9 & $25.0$ \textbf{\boldmath$(55.41{\times})$} \\
& $M=2$ & 1437.2 & $33.0$ \textbf{\boldmath$(43.51{\times})$} \\
\bottomrule
\end{tabular*}
\end{minipage}
\end{threeparttable}
\end{table}

%% file: tables/appendix/tab_training_time.tex
\begin{table}[!htbp]
\centering
\footnotesize
\begin{threeparttable}
\caption{\textbf{Training time across experimental settings.}
Entries are median per-run times in minutes on a single NVIDIA RTX A5000 GPU.}
\label{tab:app_training_time}
\setlength{\tabcolsep}{5.0pt}
\renewcommand{\arraystretch}{1.12}
\begin{tabular}{llcccc}
\toprule
\textbf{Dataset} & \textbf{Order}
& $\mathbf{W=64}$ & $\mathbf{W=128}$ & $\mathbf{W=256}$ & $\mathbf{W=512}$ \\
\midrule
\multirow{3}{*}{SHD}
& $M=0$ & 6.6 & 6.6 & 6.0 & 8.4 \\
& $M=1$ & 6.0 & 6.6 & 6.6 & 7.8 \\
& $M=2$ & 6.6 & 6.6 & 6.6 & 8.4 \\
\midrule
\multirow{3}{*}{SSC}
& $M=0$ & 51.0 & 50.4 & 52.2 & 68.4 \\
& $M=1$ & 51.6 & 53.4 & 52.2 & 69.0 \\
& $M=2$ & 50.4 & 54.0 & 52.8 & 69.0 \\
\midrule
\multirow{3}{*}{GSC}
& $M=0$ & 19.8 & 19.8 & 21.6 & 22.2 \\
& $M=1$ & 21.0 & 20.4 & 21.0 & 24.0 \\
& $M=2$ & 21.0 & 21.0 & 21.0 & 24.0 \\
\bottomrule
\end{tabular}
\end{threeparttable}
\end{table}

%% file: tables/appendix/tab_abl_gsc.tex
\begin{table}[!htbp]
\centering
\begin{minipage}{1.00\linewidth}
\centering
\footnotesize
\caption{\textbf{Factorized ablation on GSC across hidden sizes.} We compare baselines with FiTS variants initialized on a log-scale target frequency range through the inverse map. The FiTS rows separate frozen target frequencies, learnable target frequencies, and TS modules.}
\label{tab:gsc_hidden_size}
\setlength{\tabcolsep}{5pt}
\begin{tabular*}{\linewidth}{@{\extracolsep{\fill}}lcccc@{}}
\toprule
\multicolumn{1}{c}{\multirow{2}{*}[-0.5ex]{\textbf{Model}}} & \multicolumn{4}{c}{\textbf{Hidden Size}} \\
\cmidrule(lr){2-5}
& \textbf{64} & \textbf{128} & \textbf{256} & \textbf{512} \\
\midrule
\multicolumn{5}{l}{\textit{\textbf{Baselines}}} \\
\hspace{0.6em}Plain LIF
& $74.17 \pm 0.32\%$ & $78.02 \pm 0.18\%$ & $77.61 \pm 0.37\%$ & $80.02 \pm 0.16\%$ \\
\hspace{0.6em}Adapt. LIF (Frozen zero init.)
& $74.45 \pm 0.38\%$ & $78.05 \pm 0.25\%$ & $77.91 \pm 0.15\%$ & $79.89 \pm 0.19\%$ \\
\midrule
\addlinespace[3pt]
\multicolumn{5}{l}{\textit{\textbf{FiTS with log-scale target frequency init.}}} \\
\hspace{0.6em}FS (Frozen target frequency)
& $84.99 \pm 0.31\%$ & $89.72 \pm 0.15\%$ & $91.82 \pm 0.14\%$ & $92.50 \pm 0.15\%$ \\
\hspace{0.6em}FS (Learnable target frequency)
& $89.84 \pm 0.06\%$ & $92.01 \pm 0.10\%$ & $92.89 \pm 0.10\%$ & $93.27 \pm 0.12\%$ \\
\hspace{0.6em}FS + TS ($M\!=\!1$)
& $90.89 \pm 0.16\%$ & $92.95 \pm 0.10\%$ & $\mathbf{94.08 \pm 0.20\%}$ & $\mathbf{94.48 \pm 0.12\%}$ \\
\hspace{0.6em}FS + TS ($M\!=\!2$)
& $\mathbf{90.94 \pm 0.18\%}$ & $\mathbf{93.06 \pm 0.04\%}$ & $93.94 \pm 0.12\%$ & $93.96 \pm 0.14\%$ \\
\bottomrule
\end{tabular*}
\end{minipage}
\end{table}

%% file: tables/appendix/tab_abl_shd.tex
\begin{table}[!htbp]
\centering
\begin{minipage}{1.00\linewidth}
\centering
\footnotesize
\caption{\textbf{Factorized ablation on SHD$^\ast$ across hidden sizes.}
All results use the validation-split SHD protocol, where 20\% of the training set is held out for validation and test accuracy is reported at the best validation epoch. We compare baselines with FiTS variants initialized on a log-scale target frequency range through the inverse map. The FiTS rows separate frozen target frequencies, learnable target frequencies, and TS modules.}
\label{tab:shdstar_hidden_size}
\setlength{\tabcolsep}{5pt}
\begin{tabular*}{\linewidth}{@{\extracolsep{\fill}}lcccc@{}}
\toprule
\multicolumn{1}{c}{\multirow{2}{*}[-0.5ex]{\textbf{Model}}} &
\multicolumn{4}{c}{\textbf{Hidden Size}} \\
\cmidrule(lr){2-5}
& \textbf{64} & \textbf{128} & \textbf{256} & \textbf{512} \\
\midrule
\multicolumn{5}{l}{\textit{\textbf{Baselines}}} \\
\hspace{0.6em}Plain LIF
& $77.80 \pm 0.69\%$
& $78.82 \pm 0.78\%$
& $79.34 \pm 0.67\%$
& $79.70 \pm 0.35\%$ \\
\hspace{0.6em}Adapt. LIF (Frozen zero init.)
& $77.50 \pm 0.67\%$
& $78.16 \pm 0.84\%$
& $79.47 \pm 0.57\%$
& $79.47 \pm 0.36\%$ \\
\midrule
\addlinespace[3pt]
\multicolumn{5}{l}{\textit{\textbf{FiTS with log-scale target frequency init.}}} \\
\hspace{0.6em}FS (Frozen target frequency)
& $90.43 \pm 0.56\%$
& $91.55 \pm 0.39\%$
& $92.30 \pm 0.38\%$
& $92.30 \pm 0.35\%$ \\
\hspace{0.6em}FS (Learnable target frequency)
& $93.01 \pm 0.18\%$
& $93.70 \pm 0.14\%$
& $93.99 \pm 0.12\%$
& $94.02 \pm 0.21\%$ \\
\hspace{0.6em}FS + TS ($M\!=\!1$)
& $\mathbf{93.47 \pm 0.20\%}$
& $\mathbf{94.12 \pm 0.30\%}$
& $93.99 \pm 0.15\%$
& $\mathbf{94.07 \pm 0.11\%}$ \\
\hspace{0.6em}FS + TS ($M\!=\!2$)
& $93.25 \pm 0.38\%$
& $93.39 \pm 0.15\%$
& $\mathbf{94.38 \pm 0.12\%}$
& $94.04 \pm 0.24\%$ \\
\bottomrule
\end{tabular*}
\end{minipage}
\end{table}

%% file: tables/appendix/tab_abl_ssc_extend.tex
\begin{table}[!htbp]
\centering
\begin{minipage}{1.00\linewidth}
\centering
\footnotesize
\caption{\textbf{Extended TS cascade order sweep on SSC across hidden sizes.}
We extend the SSC ablation by evaluating larger TS cascade orders $M\!=\!3,4,5$. FiTS variants are initialized on a log-scale target frequency range through the inverse map. Results are reported as mean $\pm$ standard deviation across random seeds.}
\label{tab:app_ssc_m_sweep}
\setlength{\tabcolsep}{5pt}
\begin{tabular*}{\linewidth}{@{\extracolsep{\fill}}lcccc@{}}
\toprule
\multicolumn{1}{c}{\multirow{2}{*}[-0.5ex]{\textbf{Model}}} & \multicolumn{4}{c}{\textbf{Hidden Size}} \\
\cmidrule(lr){2-5}
& \textbf{64} & \textbf{128} & \textbf{256} & \textbf{512} \\
\midrule
\multicolumn{5}{l}{\textit{\textbf{Baselines}}} \\
\hspace{0.6em}Plain LIF
& $60.91 \pm 0.28\%$
& $65.91 \pm 0.16\%$
& $68.52 \pm 0.15\%$
& $69.17 \pm 0.18\%$ \\
\hspace{0.6em}Adapt. LIF (Frozen zero init.)
& $61.18 \pm 0.22\%$
& $65.96 \pm 0.14\%$
& $68.38 \pm 0.13\%$
& $68.88 \pm 0.13\%$ \\
\midrule
\addlinespace[3pt]
\multicolumn{5}{l}{\textit{\textbf{FiTS with log-scale target frequency init.}}} \\
\hspace{0.6em}FS (Frozen target frequency)
& $69.46 \pm 0.30\%$
& $73.75 \pm 0.15\%$
& $76.15 \pm 0.20\%$
& $77.13 \pm 0.22\%$ \\
\hspace{0.6em}FS (Learnable target frequency)
& $71.72 \pm 0.12\%$
& $74.94 \pm 0.16\%$
& $76.81 \pm 0.08\%$
& $77.78 \pm 0.15\%$ \\
\hspace{0.6em}FS + TS ($M\!=\!1$)
& $72.58 \pm 0.08\%$
& $75.69 \pm 0.12\%$
& $77.28 \pm 0.06\%$
& $78.22 \pm 0.12\%$ \\
\hspace{0.6em}FS + TS ($M\!=\!2$)
& $\mathbf{73.16 \pm 0.30\%}$
& $75.94 \pm 0.16\%$
& $\mathbf{77.50 \pm 0.18\%}$
& $78.23 \pm 0.16\%$ \\
\hspace{0.6em}FS + TS ($M\!=\!3$)
& $73.01 \pm 0.10\%$
& $75.93 \pm 0.13\%$
& $77.47 \pm 0.08\%$
& $78.23 \pm 0.15\%$ \\
\hspace{0.6em}FS + TS ($M\!=\!4$)
& $73.06 \pm 0.13\%$
& $\mathbf{75.97 \pm 0.21\%}$
& $77.47 \pm 0.16\%$
& $\mathbf{78.25 \pm 0.14\%}$ \\
\hspace{0.6em}FS + TS ($M\!=\!5$)
& $73.16 \pm 0.08\%$
& $75.87 \pm 0.05\%$
& $77.46 \pm 0.09\%$
& $78.12 \pm 0.10\%$ \\
\bottomrule
\end{tabular*}
\end{minipage}
\end{table}

%% file: tables/appendix/tab_freq_initialization.tex
\begin{table}[htbp]
\centering
\footnotesize
\caption{\textbf{Target-frequency initialization on SHD$^\ast$.}
Test accuracy at the best validation epoch is reported as mean $\pm$ standard deviation over five random seeds.}
\label{tab:freq_initialization}
\setlength{\tabcolsep}{10pt}
\begin{tabular}{lcc}
\toprule
\textbf{Initialization}
& \textbf{Frozen target freq.}
& \textbf{Learnable target freq.} \\
\midrule
Log-scale
& $92.30 \pm 0.67\%$
& $92.98 \pm 0.40\%$ \\
Linear-scale
& $91.65 \pm 1.02\%$
& $91.04 \pm 0.44\%$ \\
Constant $\sqrt{50}$\,Hz
& $91.08 \pm 0.83\%$
& $92.47 \pm 0.31\%$ \\
Constant $25$\,Hz
& $57.57 \pm 1.09\%$
& $70.56 \pm 0.27\%$ \\
\bottomrule
\end{tabular}
\end{table}

%% file: tables/appendix/tab_target_freq_perturb.tex
\begin{table}[!htbp]
\centering
\footnotesize
\caption{\textbf{Inference-time target frequency perturbation on SSC.}
For each hidden width, we evaluate 5 FS-only ($M\!=\!0$) checkpoints.
Each cell reports test accuracy, with the paired change from the corresponding unperturbed checkpoint shown in parentheses.}
\label{tab:app_target_freq_perturb}
\setlength{\tabcolsep}{5pt}
\begin{tabular*}{\linewidth}{@{\extracolsep{\fill}}c c c c@{}}
\toprule
$W$
& Full
& Reset $f^\star$
& Shuffle $f^\star$ \\
\midrule
64
& $71.72 \pm 0.12\%$
& $20.95 \pm 6.80\%$ \textbf{\boldmath($-50.78 \pm 6.78\%$)}
& $3.69 \pm 1.13\%$ \textbf{\boldmath($-68.03 \pm 1.12\%$)} \\[2pt]
128
& $74.94 \pm 0.16\%$
& $44.39 \pm 7.46\%$ \textbf{\boldmath($-30.55 \pm 7.51\%$)}
& $3.49 \pm 1.12\%$ \textbf{\boldmath($-71.45 \pm 1.10\%$)} \\[2pt]
256
& $76.81 \pm 0.08\%$
& $69.86 \pm 0.72\%$ \textbf{\boldmath($-6.95 \pm 0.74\%$)}
& $4.20 \pm 0.48\%$ \textbf{\boldmath($-72.61 \pm 0.44\%$)} \\[2pt]
512
& $77.78 \pm 0.15\%$
& $73.43 \pm 0.49\%$ \textbf{\boldmath($-4.34 \pm 0.43\%$)}
& $3.59 \pm 0.34\%$ \textbf{\boldmath($-74.19 \pm 0.30\%$)} \\
\bottomrule
\end{tabular*}
\end{table}

%% file: tables/appendix/tab_smnist.tex
\begin{table}[!htbp]
\centering
\footnotesize
\begin{threeparttable}
\setlength{\tabcolsep}{8.5pt}
\renewcommand{\arraystretch}{1.08}
\caption{\textbf{sMNIST results.}
We compare FiTS with representative previously reported models on sequential
MNIST. Parameter counts are reported as given in the corresponding references.}
\label{tab:app_smnist}
\begin{tabular}{c l c c}
\toprule
\textbf{Dataset} & \textbf{Method} & \textbf{\# Params} & \textbf{Acc.} \\
\midrule
\multirow{6}{*}[-1.5ex]{sMNIST}
& TC-LIF~\cite{zhang2024tc}        & 0.09\,M & 97.35\% \\
& S5-RF~\cite{huber2024scaling}    & 0.04\,M & 98.89\% \\
& DH-SRNN~\cite{zheng2024dendritic} & 0.08\,M & 98.90\% \\
& D-RF~\cite{zhang2026dendritic}   & 0.16\,M & 99.50\% \\
& PMSN~\cite{chen2024pmsn}         & 0.16\,M & 99.53\% \\
& PRF~\cite{huang2024prf}          & 0.07\,M & 99.18\% \\
& \cellcolor{gray!15}\textbf{FiTS (ours, \(M=1\), \(W=256\))}
& \cellcolor{gray!15}\textbf{0.07\,M}
& \cellcolor{gray!15}\textbf{98.28\%} \\
\bottomrule
\end{tabular}
\end{threeparttable}
\end{table}

%% file: tables/appendix/tab_dvsgesture.tex
\begin{table}[!htbp]
\centering
\footnotesize
\begin{threeparttable}
\begin{minipage}{0.78\linewidth}
\caption{\textbf{Non-auditory event-based check on DVS Gesture 128.}
All models use the same lightweight feedforward convolutional architecture and the official train/test split. For convolutional layers, FiTS parameters are shared across spatial locations within each channel. We report final-epoch test accuracy across 5 random seeds under a fixed training schedule.}
\label{tab:app_dvsgesture}
\setlength{\tabcolsep}{6pt}
\renewcommand{\arraystretch}{1.10}
\begin{tabular*}{\linewidth}{@{\extracolsep{\fill}}lcc@{}}
\toprule
\textbf{Model} & \textbf{\# Params} & \textbf{Test Acc. (\%)} \\
\midrule
Plain LIF & $0.0726$M & $87.01 \pm 1.34$\% \\
FS-only ($M=0$) & $0.0728$M & $89.72 \pm 0.84$\% \\
FS+TS ($M=1$) & $0.0732$M & $90.00 \pm 0.86$\% \\
FS+TS ($M=2$) & $0.0735$M & $\mathbf{90.35 \pm 1.08\%}$ \\
\bottomrule
\end{tabular*}
\end{minipage}
\end{threeparttable}
\end{table}